\documentclass[letterpaper]{article} 
\usepackage{aaai2026}  
\usepackage{times}  
\usepackage{helvet}  
\usepackage{courier}  
\usepackage[hyphens]{url}  
\usepackage{graphicx} 
\usepackage{natbib}  
\usepackage{caption} 
\usepackage{fontawesome5}
\definecolor{iconC}{HTML}{2B2B2B}
\newcommand{\levelicon}[1]{\raisebox{-0.55\height}{\textcolor{iconC}{\fontsize{18pt}{18pt}\selectfont #1}}}

\usepackage{enumerate}

\usepackage{stmaryrd}

\usepackage{csquotes}

\usepackage{amsmath,amsthm}
\usepackage{amsfonts}

\usepackage{thmtools,thm-restate}
\usepackage{tikz}
\usetikzlibrary{positioning}

\usepackage[most]{tcolorbox}
\usepackage{array}

\usepackage{booktabs, xcolor, colortbl, tabularx}

\newtcolorbox{policybox}{
  colback=gray!10,
  colframe=gray!40,
  boxrule=0.5pt,
  arc=2pt,
  left=6pt,
  right=6pt,
  top=6pt,
  bottom=6pt
}

\newtcolorbox{examplebox}{
  colback=blue!10,
  colframe=blue!40,
  boxrule=0.5pt,
  arc=2pt,
  left=6pt,
  right=6pt,
  top=6pt,
  bottom=6pt
}

\usepackage{xcolor,colortbl}

\newtheorem{theorem}{Theorem}

\theoremstyle{definition}
\newtheorem{definition}{Definition}
\newtheorem{counterexample}{Counterexample}
\usepackage{cleveref}

\title{Explainable AI Isn't Enough! Rethinking Algorithmic Contestability}

\author{
    Timo Freiesleben\textsuperscript{\rm 1,2,3},
    Kristof Meding\textsuperscript{\rm 4,5},
    Gunnar König\textsuperscript{\rm 4,6}
}

\affiliations{
    \textsuperscript{\rm 1}LMU Munich, Germany\\
    \textsuperscript{\rm 2}Munich Center for Machine Learning (MCML)\\
    \textsuperscript{\rm 3}Munich Center for Mathematical Philosophy (MCMP)\\ 
    \textsuperscript{\rm 4}University of Tübingen, Germany\\
        \textsuperscript{\rm 5}
    CZS Institute for Artificial Intelligence and Law\\
    \textsuperscript{\rm 6}T\"ubingen AI Center\\
    timo.freiesleben@lmu.de
}

\begin{document}

\maketitle

\begin{abstract}
Machine learning systems increasingly make life-changing decisions about individuals, such as loan approvals, hiring, and cheating detection, raising a pressing question: how can individuals respond to negative decisions made by these opaque systems? While explainable artificial intelligence (XAI) has largely focused on algorithmic recourse—helping individuals change their features to obtain a desired outcome—the parallel problem of algorithmic contestability—helping individuals review and correct erroneous algorithmic decisions—has received far less attention, despite its central ethical and legal importance. We trace this neglect to the absence of clear formal definitions and a systematic operationalization of contestability as an algorithmic problem. To address it, we propose an operational definition of contestability as a natural complement to recourse: contestability starts from the presumption that a decision may be incorrect and focuses on identifying evidence to challenge and potentially overturn it, whereas recourse assumes the decision is valid and instead provides pathways for changing it. We show that standard XAI explanations, such as counterfactuals, LIME, or Anchors, even when combined with human intuitions about decision continuity or monotonicity, reveal only errors in the neighborhood of the individual, but provide insufficient grounds for overturning the decision at hand. Going thus beyond traditional XAI, we identify three types of evidence warranting reversal according to the decision maker’s own ethical standards: predictive multiplicity, incorrect feature values, and neglected overruling evidence. We argue that these render decisions normatively indefensible and thus successfully contestable. Finally, we analyze how existing EU legislation connects to our framework and argue that individuals already hold some legal rights to these forms of evidence.
\end{abstract}

\section{Introduction}
Machine learning systems increasingly make life-changing decisions about individuals. They influence whether laypeople receive loans, whether applicants are invited to job interviews, or even whether students are accused of cheating on an exam. Seeing one’s life placed in the hands of an opaque algorithm can make people feel powerless: What can they do when confronted with an unfavorable decision?

Explainable artificial intelligence has long been motivated, at least in part, by the aim of restoring the agency of individuals who face unfavorable algorithmic decisions \citep{lipton2018mythos,wachter2017counterfactual}. Famously, \emph{algorithmic recourse explanations} have been proposed to recommend actionable steps, e.g., paying off outstanding loans, that enable individuals to reverse an unfavorable decision. Another proposed way to restore individuals’ agency is to provide them with grounds for \emph{contesting} adverse decisions, that is, to provide individuals with information to review the decision-making process and intervene in situations where they feel judged incorrectly or unfairly. For example, individuals may contest a rejected credit application if their rejection was based on living in a neighborhood with high loan default rates; whether the contestation is successful and leads to a reversal of the decision is then subject to a negotiation process.

While the literature on algorithmic recourse has become a cornerstone of XAI research, with both solid conceptual as well as algorithmic foundations \citep{ustun2019actionable,venkatasubramanian2020philosophical,konig2023improvement}, the literature on algorithmic contestability has remained comparatively underexplored. This is striking given contestability arguably has at least as great ethical and legal importance as algorithmic recourse \citep{selbst2018intuitive,lyons2021conceptualising,bayamliouglu2022right}. The reasons for this disparity are undoubtedly multifaceted: contestability is deeply intertwined with legal questions that vary across jurisdictions; technical researchers may feel uncomfortable engaging with legal concepts; and decision makers often have little incentive to provide individuals with evidence that could be used to challenge their decisions. However, we believe there is a more fundamental reason why contestability has been largely neglected as an algorithmic problem: the absence of clear formal definitions and a systematic operationalization of contestability itself. The goal of this paper is to provide precisely these. 

We conceptualize algorithmic contestability as the counterpart to algorithmic recourse, with both approaches ultimately aimed at helping individuals achieve their desired outcomes (Section~\ref{sec:AlgorithmicDecisionMaking}). In the context of recourse, individuals trust the correctness of the model’s decision and use recourse recommendations to obtain the desired outcome. That is, they believe that the decision about the individual (e.g., loan rejection) aligns with the true state (i.e., the loan would not be repaid) and the recourse explanation provides information that enables them to change both the decision and the true state, that is, the loan is approved and ultimately repaid. Algorithmic contestability takes the opposite perspective: the individual maintains that the decision (e.g., loan rejection) is not aligned with the true state (i.e., the loan would have been repaid) and therefore seeks information that allows them to alter the decision to match the true state, so that they receive the loan.

We show that common XAI methods, such as counterfactual explanations \citep{wachter2017counterfactual}, Local Interpretable Model-Agnostic Explanations (LIME) \citep{ribeiro2016should}, and Anchors \citep{ribeiro2018anchors}, are of limited value for algorithmic contestability. While they can reveal conflicts with human intuitions concerning continuity and monotonicity, and thereby point to errors somewhere in the individual's neighborhood, they do not provide sufficient grounds to overturn a decision (\Cref{sec:XAI}). As a response, we move beyond XAI and investigate other types of evidence that can enable contestability from a \emph{normative perspective}, that is, evidence that should lead decision makers to change their decision according to their own ethical norms (\Cref{sec:EpistemicReasons}). The three types of evidence are 1. \emph{Predictive multiplicity}: conflicting predictions for the individual from a set of models with equal empirical and theoretical support; 2. \emph{Incorrect feature values:} wrong information about the individual that if corrected would lead to a different decision; 3. \emph{Neglected overruling evidence:} evidence about other features of the individual that if taken into account would overturn the decision. We show that these forms of evidence can be sufficient to alter algorithmic decisions on normative grounds. We conclude by interpreting the European General Data Protection Regulation (GDPR) in light of our framework and argue that individuals already possess rights to these forms of evidence under the existing legislative landscape \Cref{sec:Legal}.

\section{Related work}
\label{sec:related}

Contestability is a core ethical requirement in algorithmic decision-making and, consequently, a central element in several worldwide legislations \citep{vaccaro2019contestability,maxwell2023meaningful}. For example, it is one of the eight guiding principles of Australia’s AI Ethics Framework \citep{AI_Ethics_Framework_2019}, lies at the “heart of legal rights that afford individuals access to personal data and insight into decision-making processes” \citep{mulligan2019shaping}, and constitutes the “backbone” of the widely discussed Article 22 of the GDPR \citep{bayamliouglu2022right}. The goals and nature of contestability are, however, multifaceted and often debated, which makes both its legal and technical implementation challenging \citep{lyons2021conceptualising}. It is widely agreed that algorithmic contestability remains under-researched, especially from a technical perspective \citep{lyons2021conceptualising,alfrink2023contestable,maxwell2023meaningful,yurrita2025identifying}.

\citet{henin2021framework} provide a norm-based formal framework of contestability that allows decisions to be contested by reference to norms and justifications. Their formal model, \emph{Algocate}, provides an interactive tool for generating evidence in favor of and against a decision. While their idea of building upon norms and justifications is conceptually sound \citep{aler2020contestable}, their technical framework is difficult to implement in practice: it requires a complete logical specification of decision norms as well as an ordering relation over argument strengths. More recently, \citet{moreira2025explainable} attempted to formalize contestability in a comparatively broad sense, encompassing the facilitation of risk analysis, the assurance of accountability, and the enablement of decision reversal. Unlike this work, we focus exclusively on contestability to enable decision reversal and provide a concrete framework together with workable definitions for this specific scenario. Since we narrow our focus to enabling decision reversal through contestability, we abstract away from other aspects in the construction of algorithmic systems that may be contestable \citep{hirsch2017designing,sarra2020put,alfrink2023contestable,yurrita2025identifying}.

While providing individuals with grounds to contest individual decisions is often cited as a central motivation of the field of explainable artificial intelligence (XAI) \citep{ribeiro2016should,lipton2018mythos,adadi2018peeking,molnar2020interpretable}, only very few technical works address contestability in greater detail \citep{mansi2025legally}. One notable exception is the work on counterfactual explanations by \citet{wachter2017counterfactual}. However, while \citet{wachter2017counterfactual} motivate \emph{why} counterfactuals can enable individuals to contest decisions and satisfy the legal requirements established by the GDPR, they do not demonstrate \emph{how} their method can be tailored to this purpose, nor under which conditions contestation based on counterfactuals should be successful. Other explanation methods, such as LIME or Shapley values, are also commonly discussed as potentially fostering contestability \citep{mulligan2019shaping,yurrita2023disentangling}. Many ethical and legal scholars argue that XAI could serve as an “enabler for contestability” \citep{selbst2018intuitive,lyons2021conceptualising,vredenburgh2022right,maxwell2023meaningful,schmude2025two}. Yet, they often remain vague about \emph{how} exactly XAI facilitates such contestation. The most detailed account is offered by \citet{selbst2018intuitive}, who argue that explanations must be paired with human intuitions to support contestability. However, as we will demonstrate, the insights that common XAI methods provide into the (in)correctness of decisions are quite limited, even when individuals actively draw on their own intuitions.

A closely related problem to algorithmic contestability is the problem of algorithmic recourse \citep{ustun2019actionable}. Both algorithmic recourse and contestability aim to reverse decisions; however, in recourse, the individual perceives the decision as correct, whereas in contestability, it is perceived as incorrect \citep{venkatasubramanian2020philosophical}. The recourse literature has evolved significantly over the years \citep{karimi2022survey}, with work investigating the robustness \citep{upadhyay2021towards,konig2025performative}, fairness \citep{von2022fairness}, and privacy risks \citep{pawelczyk2023privacy} of recourse. We hope that our definitions and operationalizations of algorithmic contestability will enable a similar formal treatment of contestability, allowing many ideas and methods to be transferred and ultimately placing algorithmic contestability on equal footing with algorithmic recourse.

\section{From algorithmic decision-making to algorithmic contestability}
\label{sec:AlgorithmicDecisionMaking}
Algorithmic decision-making is characterized by a few core components. Individuals possess \emph{unobserved properties} that decision makers would like to know, but that are difficult or sometimes impossible to measure directly. For example, in lending, credit institutions want to know whether individuals will ultimately repay their loans, and in cheating detection, schools want to know whether a student has cheated. These unobserved properties lie at the heart of decision makers' \emph{core ethical principles}\footnote{To support readers with more technical backgrounds, \Cref{app:philGloss} offers a short glossary explaining the philosophical concepts used throughout the paper.}  that guide their decisions. Ethical does not mean that these principles are justified morally but only that they are maxims the decision maker commits to. For instance, banks aim to grant loans only to individuals who will be able to repay them, and schools aim to fail students who have cheated. Ideally, decision makers would make decisions strictly in accordance with these core ethical principles; we call such decisions \emph{normatively correct}.

\begin{definition}
Let $i$ be an individual with an unobserved binary attribute $y^{(i)} \in \lbrace 0,1\rbrace$. We define the \emph{normatively correct} decision as $d^*(i):=y^{(i)}$.
\end{definition}

However, because these properties are unobserved, decision makers must infer them from \emph{features}, i.e., measured properties that are highly correlated with the unobserved property of interest. Examples include age, salary, or credit utilization in lending, or a student’s submitted answers in cheating detection. These features allow decision makers to predict the unobserved property with some probability, but such predictions are never perfect. For instance, although a young person with a high salary and low credit utilization is very likely to repay her loan, she may still default due to unforeseen circumstances such as job loss or an accident. Decisions based on such predictions therefore require decision makers to accept a certain \emph{inductive risk} \citep{douglas2000inductive}. For example, they may choose to reject loan applicants whose probability of repayment falls below a given threshold. Accepting inductive risk entails accepting the possibility of normatively incorrect decisions: some applicants will be approved despite ultimately defaulting, and some students will be accused of cheating despite having done nothing wrong. Setting the level of inductive risk is, ultimately, an ethical decision that decision makers must make \citep{douglas2009science}. We call decisions that adhere to the chosen inductive risk \emph{epistemically correct}.

\begin{definition}
Let $\tau \in [0,1]$ denote the inductive-risk threshold, and let $p(x^{(i)}) := \mathbb{P}(Y = 1 \mid X = x^{(i)})$ be the probability that individual $i$ possesses the unobserved property $Y=1$ given her features $x^{(i)}$. We define the \emph{epistemically correct} decision given features $X$ as
\begin{equation*}
d(x^{(i)}) := \begin{cases}
1 \quad \text{if } p(x^{(i)}) > \tau, \\
0 \quad \text{else.}
\end{cases}
\end{equation*}
\end{definition}

Decision makers do not have access to the true probabilities of loan default or cheating conditional on feature values. Instead, they must estimate these probabilities from data, typically using machine learning methods. The machine learning model selection process involves a variety of choices—concerning model classes, regularization techniques, hyperparameters, and more (this pipeline is discussed in more detail in \Cref{subsec:multiplicity})—with the goal of training a model on a dataset that achieves high predictive performance. Moreover, decision makers usually do not have access to individuals’ true feature values. In lending, for example, individuals may provide incorrect information, or lenders may store inaccurate data. 

\begin{definition}
Let $\hat{p}$ denote the model resulting from this model selection process, and let $\tilde{x}^{(i)}$ denote the feature values available to the decision maker. We define the \emph{actual decision} as
\begin{equation*}
\hat{d}(\tilde{x}^{(i)}) := \begin{cases}
1 \quad \text{if } \hat{p}(\tilde{x}^{(i)}) > \tau, \\
0 \quad \text{else.}
\end{cases}
\end{equation*}
\end{definition}

Algorithmic \emph{decisions} thus result from two components: decision makers have ethical norms that determine which decisions are (normatively or epistemically) correct according to their core ethical principles and chosen inductive risk, and these ethical norms yield actual decisions only through their combination with epistemic beliefs. In the context of machine learning, these epistemic beliefs depend on measured feature values and on the prediction model induced by a learning algorithm.

\begin{table*}[t]
\centering
\begin{minipage}[c]{0.94\textwidth} 
  \begin{tabularx}{\textwidth}{@{}c@{\hspace{8pt}} p{2.7cm} p{5.1cm} X}
  \toprule
   & \textbf{Level} & \textbf{Meaning} & \textbf{Evidence} \\
  \midrule
  \levelicon{\faBalanceScaleLeft}
  & Normative \newline contestability 
  & The decision $\hat{d}(x^{(i)})$ misjudges the actual label $y^{(i)}$.
  & Overruling evidence (\Cref{subsec:Overruling}) \\
  \addlinespace
  \levelicon{\faSearch}
  & Epistemic \newline contestability 
  & The decision $\hat{d}(x^{(i)})$ misjudges what the features warrant $d(x^{(i)})$.
  & Incorrect feature values (\Cref{subsec:WrongFeatureValues}), \newline Predictive multiplicity (\Cref{subsec:multiplicity})
     \\
  \addlinespace
  \levelicon{\faStreetView}
  & Somewhere \newline contestability 
  & Some nearby $x$ is epistemically contestable.
  & Reason intuitions + Anchors (\Cref{subsec:anchors}), Continuity intuitions + Counterfactuals (\Cref{subsec:counterfactuals})
     \\
  \addlinespace
  \levelicon{\faChartArea}
  & Somewhere \newline inaccuracy 
  & The probability estimate $\hat{p}$ is off somewhere nearby.
  & Monotonicity intuitions + LIME (\Cref{subsec:LIME}) \\
  \bottomrule
  \end{tabularx}
\end{minipage}
\begin{minipage}[c]{0.05\textwidth}
  \centering
  \begin{tikzpicture}[scale=1]
    \draw[<-, thick] (0,0) -- (0,-4.5) 
      node[midway, rotate=90, below] {stronger contesting grounds};
  \end{tikzpicture}
\end{minipage}
\caption{\textbf{Levels of contestability.} The table summarizes our main contributions in this paper: 1) We introduce four levels of algorithmic contestability, namely normative, epistemic, somewhere contestability, and local inaccuracies; 2) We show that XAI explanations, even when combined with human intuitions, only allow one to reach the lowest two levels; 3) We show that other types of evidence, such as predictive multiplicity, support higher levels of contestability. The arrow on the side indicates that establishing higher levels of contestability provides stronger normative grounds for overturning a decision.}
\label{tab:levels}
\end{table*}

That a decision is normatively or epistemically correct in our terminology does not mean that it is uncontroversial, only that it aligns with the decision maker’s ethical norms. One may, for instance, contest a school’s core ethical principle and argue that using generative AI in exams should not be penalized \citep{anders2023using}. Likewise, decisions that are epistemically correct can still be challenged on ethical grounds, for example, by arguing that a threshold-based policy in cheating detection involves an unacceptably high inductive risk of false accusations. In this work, however, we focus not on challenging ethical norms but on challenging the epistemic beliefs underlying algorithmic decisions. Which decisions should be contestable on such grounds?

\begin{definition}
\label{def:contestable}
Assume a decision $\hat{d}(\tilde{x}^{(i)})$ is made about an individual $i$. We call the decision
\begin{itemize}
\item \emph{normatively contestable} if the decision is normatively incorrect, i.e., $\hat{d}(\tilde{x}^{(i)})\neq d^*(i)$, and
\item \emph{epistemically contestable} relative to a feature set $X$ if the decision is epistemically incorrect, i.e., $\hat{d}(\tilde{x}^{(i)})\neq d(x^{(i)})$.
\end{itemize}
\end{definition}

Normative contestability binds decision makers to their own ethical principles. For example, if a student is accused of cheating but did not in fact cheat, she should be able to contest and overturn the decision as a matter of moral principle. Epistemic contestability, by contrast, binds decision makers to the inductive risk they have accepted. For instance, if a decision maker assigns a high default risk to an individual who is in fact low risk, the individual should be able to challenge and reverse the decision on epistemic grounds. 

Anchoring contestability in decision makers’ own principles has a key advantage: most disputes over algorithmic decisions are unlikely to reach the courts and are instead resolved through negotiation between the individual and the decision maker. If decision makers remain committed to their ethical and epistemic principles, then decisions that are normatively or epistemically contestable should, provided individuals can present sufficient supporting evidence, be overturned.

\section{XAI explanations can only establish that errors occur somewhere}
\label{sec:XAI}
Model-agnostic XAI techniques provide insight into a machine learning model by analyzing its sensitivity to changes in the input \citep{scholbeck2019sampling}. Prominent examples include counterfactual explanations, LIME, or Anchors.

Why do scholars regard XAI as a key enabler of algorithmic contestability? This belief is rooted in an idea that has been articulated by \citet{selbst2018intuitive} and is illustrated in \Cref{fig:violations}: suppose an explanation method, such as LIME, indicates that the requested loan amount has a locally positive effect on the model’s prediction, implying that requesting a higher loan increases the likelihood of approval. An individual may find this relationship highly counterintuitive, expecting instead that requesting a larger loan would reduce her chances of obtaining it. In this way, the combination of \emph{human intuitions} and \emph{model explanations} can reveal potential model errors and thereby seem to provide grounds for contestation. In the following, we show that while this line of reasoning is not misguided, the grounds for contestation it yields are generally weak.

But before we do so, we must clarify what we mean by \emph{human intuitions}. Rather than estimating conditional probabilities over unobserved variables, as machine learning models do, humans typically rely on \emph{rules of thumb} when forming beliefs and making decisions \citep{gigerenzer2011heuristic,selbst2018intuitive}. These rules are not arbitrary; they often exhibit high validity, having emerged through repeated experience and pattern recognition \citep{gigerenzer2023intelligence}. Both experts and non-experts draw on a rich repertoire of such intuitive rules, including the following: if a person highly similar to me obtains the loan, then so should I (\emph{continuity rule}); a person with a higher income than me should, all else equal, have a higher probability of obtaining the loan (\emph{monotonicity rule}); and certain properties of mine constitute sufficient or insufficient reasons for granting or denying the loan (\emph{reasons rule}). We call a human intuition \emph{correct} if the rule it expresses applies in the given context.

\begin{figure*}[t]
    \centering
    \includegraphics[width=0.95\linewidth]{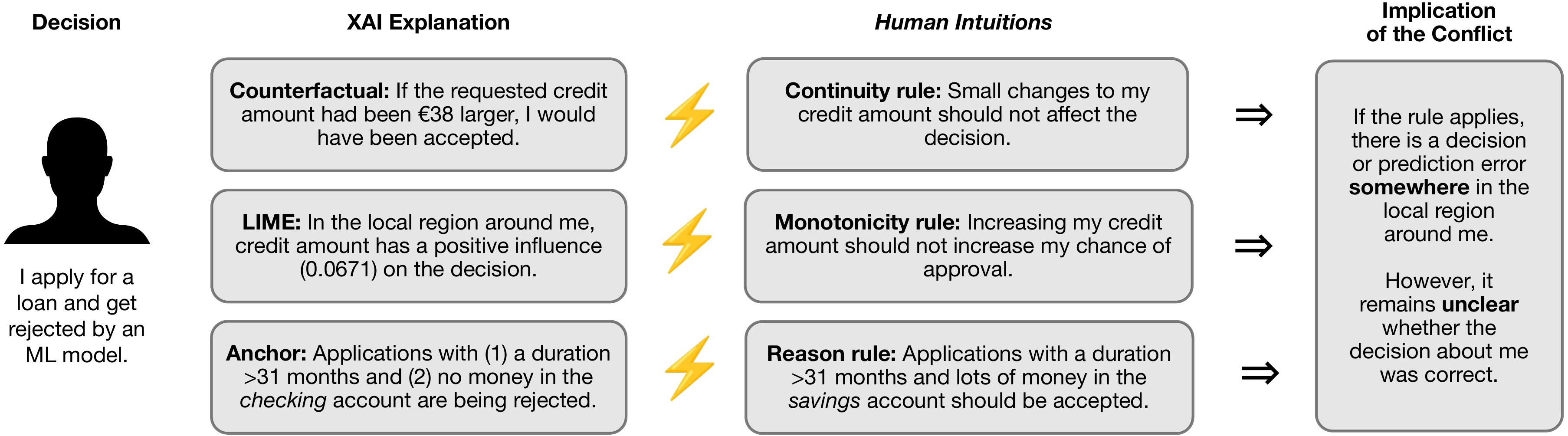}
    \caption{\textbf{XAI explanations may conflict with human intuitions.} For example, (1) a counterfactual explanation may conflict with continuity rules, (2) a LIME explanation may conflict with monotonicity rules, and (3) anchors may contradict reason rules. Even under the strong assumption that these intuitions are correct, i.e. the rules apply, such conflicts imply an error only somewhere in the neighborhood of the explained observation and thus insufficient grounds to overturn the decision at hand. Details on the examples are reported in Appendix \ref{app:details-examples}.}
    \label{fig:violations}
\end{figure*}

Suppose that an XAI explanation conflicts with human intuitions, as in the example above. To better understand what can be inferred from such a conflict, we introduce two notions of error. \emph{Somewhere contestability} refers to the existence of epistemically incorrect decisions within a neighborhood of the individual, whereas \emph{somewhere inaccuracy} refers to the existence of incorrect model predictions in such a neighborhood. More formally:
\begin{definition}
\label{def:somewhere}
Let $i$ be an individual and $\mathcal{B}_{\epsilon}(\tilde{x}^{(i)}):=\lbrace z\in \mathcal{X}\mid\|z-\tilde{x}^{(i)}\|<\epsilon\rbrace$ be an $\epsilon$-ball around its feature values for a given $\epsilon>0$. We call a decision \emph{somewhere contestable} if there exists a $x\in \mathcal{B}_{\epsilon}(\tilde{x}^{(i)})$ such that $\hat{d}(x) \neq d(x)$. We call a machine learning model $\hat{p}$ \emph{somewhere inaccurate} if there exists an $x\in \mathcal{B}_{\epsilon}(\tilde{x}^{(i)})$ such that $\hat{p}(x) \neq p(x)$.
\end{definition}

Unlike epistemic and normative contestability, somewhere contestability and somewhere inaccuracies are insufficient to overturn a decision. The presence of an epistemically incorrect prediction somewhere in the neighborhood of an individual remains fully compatible with the decision for that individual being correct. Logically, this yields a hierarchy: epistemic contestability is stronger than somewhere contestability, which in turn is stronger than somewhere inaccuracy.\footnote{We provide a short proof for this claim in the Appendix \ref{theo:Epistemic-weakly}.} \Cref{tab:levels} provides an overview of all the different levels of contestability. Although errors somewhere in the neighborhood are a necessary condition for epistemically incorrect decisions and may thus serve as a starting point for algorithmic contestation, they provide only a very weak signal: most predictive models exhibit errors somewhere, leaving it largely indeterminate whether the specific decision at hand ought to be overturned.

In the following, we show that a conflict between human intuitions and XAI explanations—even in the best-case scenario where those intuitions are correct—supports at most an inference to an error somewhere. When the intuitions are themselves incorrect, which happens regularly \citep{tversky1974judgment}, XAI explanations do not support any inference about the correctness of the decision of interest.

\subsection{Detecting continuity conflicts with counterfactual explanations}
\label{subsec:counterfactuals}
Continuity rules take the form: \emph{If two individuals have highly similar feature values, they should be judged similarly.} Such rules are widely adopted, including in the context of fairness \citep{kusner2017counterfactual}, as humans generally expect decisions to be made consistently \citep{kolodner2014case,selbst2018intuitive}.
\begin{definition}[Continuity rule]
    If person $i$ and $j$ are very similar in their measured features $\tilde{x}^{(i)}$, they should receive the same decision, i.e. $d(\tilde{x}^{(i)})=d(\tilde{x}^{(j)})$.
\end{definition} 
However, discontinuities can be justified. Consider two highly similar loan applicants, where one is 17 and the other is 18 years old. Granting the loan to the latter but not the former is justified, as they differ in legal capability, which is relevant in lending contexts.

To check whether algorithmic decision-making conflicts with continuity rules, individuals need access to the feature values of similar alternative individuals who receive a different decision. We believe that counterfactual explanations are the best available approach to generate and select such alternative feature values.\footnote{Closest real instances receiving a different classification would also be a reasonable choice, but raise data privacy concerns \citep{pawelczyk2023privacy}.} A counterfactual identifies feature values that are maximally similar to those of the individual but would have resulted in a different decision. 

\begin{definition}[Counterfactual explanation] Let $x\in\mathcal{X}$, then a counterfactual $x_c$ is an alternative point in $\mathcal{X}$ that lies (maximally)\footnote{While the original formulation of counterfactuals by \citet{wachter2017counterfactual} requires the counterfactual to be the closest input across the decision boundary, most methods generate counterfactuals heuristically, e.g. using evolutionary algorithms \citep{dandl2020multi}, and therefore produce close rather than closest counterfactuals. For purposes of contestability, maximal proximity is less essential than it is for recourse \citep{venkatasubramanian2020philosophical,freiesleben2022intriguing}.} close to $x$ according to the distance function $\|\cdot\|$ but receives a different decision, that is, $\hat{d}(x)=1$ and $\hat{d}(x_c)=0$ or vice versa. 
\end{definition}

The following theorem (proof in \Cref{app:proofs}) shows that counterfactuals can provide sufficient information for somewhere contestability, however, as the Counterexample~\ref{ce:counterfactuals} shows, they do not render \emph{the decision in question} epistemically contestable.

\begin{restatable}{theorem}{retheocount}
    \label{theo:Count}
    If the continuity rule applies to the individual and her counterfactual, and their distance is smaller than $\epsilon$ according to $\|\cdot\|$, the decision is somewhere contestable.
\end{restatable}

\subsection{Detecting monotonicity conflicts with LIME}
\label{subsec:LIME}
Monotonicity rules take the form: \emph{If an individual differs from another individual only on a feature that has a positive (or negative) relationship with the unobserved target, the former should not be judged worse or respectively better than the latter.} Such rules are widely adopted, as humans expect decisions to be affected by feature changes in a consistent direction \citep{birnbaum2019violations,selbst2018intuitive}.

\begin{definition}[Monotonicity rule]
    If an individual $i$ differs from an individual $j$ only in feature $k$ with $\tilde{x}^{i}_k\geq \tilde{x}^{j}_k$, then person $i$ should receive at least as favorable a probability as $j$, i.e. $p(\tilde{x}^{(i)}) \geq p(\tilde{x}^{(j)})$ (respectively unfavorable $p(\tilde{x}^{(i)}) \leq p(\tilde{x}^{(j)})$ for negative monotonicity in $k$).
\end{definition}

Non-monotonic decision-making can be justified, as many relationships between features and unobserved variables are themselves non-monotonic. For example, increasing income will usually improve creditworthiness, but exceptionally high incomes may correlate with volatile self-employment or irregular cash flow, which can legitimately lower repayment probability.

To check whether algorithmic decision-making conflicts with monotonicity rules, individuals need access to local feature-attribution information that identifies which features the model used and in which direction they influenced the prediction locally. One possible type of explanation to approximate such local influence directions is LIME. A LIME explanation estimates the linear contribution of each feature to the model’s decision in the neighborhood of the individual. 

\begin{definition}[LIME Explanation]
    Let $x\in\mathcal{X}$, then LIME collects the function values of the model $\hat{p}$ on points sampled from a local distribution $\pi_x$ around $x$. A linear regression is then performed on the collected samples and function values, resulting in a linear model $g(x) =\sum\limits_{i=1}^p w_i \cdot x_i$ approximating $\hat{p}$ in the local environment $\pi_x$. The local faithfulness of $g$ to $\hat{p}$ on $\pi_x$ is given by
    \begin{equation*}
        \text{LF}(g,\hat{p}):=\mathbb{E}_{\pi_x}[(g(\pi_x)-\hat{p}(\pi_x))^2].
    \end{equation*}
\end{definition}

The following theorem (proof in \Cref{app:proofs}) shows that LIME explanations, given LIME coefficients $w_k$ point in the opposite direction of monotonicity rules, only provide sufficient information to establish somewhere inaccuracy. However, as the counterexample~\ref{ce:LIME} shows, LIME explanations are insufficient to render decisions (somewhere) epistemically contestable.

\begin{restatable}{theorem}{retheolime}
\label{theo:LIME}
Let $\pi$ be the local environment\footnote{
For notational simplicity, we abbreviate the local environment here as $\pi$ rather than $\pi_{{\tilde{x}^{(i)}}}$.
} around $\tilde{x}^{(i)}$, with $\pi=(\pi_1,\dots,\pi_r)$ jointly independent random variables, and let $\pi$ be constrained by $\mathbb{P}_{\pi}(|\pi-\tilde{x}^{(i)}|>\epsilon)=0$. Assume that the monotonicity intuition applies positively (respectively negatively) to feature $k$ and that the LIME explainer yields $w_k<0$ (respectively $w_k>0$). If the local faithfulness of $g$ satisfies $\text{LF}(g,\hat{p})< \mathbb{V}_{\pi_k}[w_k \pi_k]$, then the model is somewhere inaccurate.
\end{restatable}

Other, more direct, explanation methods for detecting monotonicity conflicts include individual conditional expectation curves \citep{goldstein2015peeking} or prediction gradients \citep{wang2024gradient}, both of which also allow for the detection of somewhere inaccuracies. We chose LIME here to illustrate monotonicity conflicts because it is most commonly used in contestability-relevant domains such as lending or hiring.

\subsection{Detecting reason conflicts with Anchors}
\label{subsec:anchors}
Reason\footnote{This reflects a specific conception of reasons that ties them to sufficiency for particular decisions. An alternative view treats reasons as factors that push a decision in a given direction \citep{snedegar2018reasons}. On this account, good reasons are factors that, when present, increase the probability of the unobserved variable of interest toward the realized decision. However, this view faces a key difficulty: what counts as a good reason is often highly context-dependent and sensitive to specific counterfactual feature values, making it hard to formulate as a general rule. Moreover, even if suitable local rules could be specified, the most promising auditing method, counterfactual Shapley values \citep{albini2022counterfactual}, would still be inadequate, as it relies heavily on model predictions for data points far from the instance being explained and thus cannot even reliably detect somewhere inaccuracies.}
rules take the form: \emph{If a person has certain feature values on a subset of relevant features, this provides sufficient grounds for making the decision.} Humans often justify their decisions by providing \emph{reasons}, aspects of the specific case that motivated the decision \citep{raz1975reasons}. These reasons can then be challenged by questioning their validity or by highlighting other reasons that override them. 

\begin{definition}[Reason rule]
If an individual $i$ has feature values satisfying conditions $A$, then this provides sufficient grounds to make the decision $d(\tilde{x}^{(i)})=1$ rather than $d(\tilde{x}^{(i)})=0$, or vice versa. 
\end{definition}

Some reason rules adopted by humans are in fact incorrect. For example, while many people assume that a very high salary suffices to justify a positive credit decision, substantial debt or an unstable job history may outweigh a high salary, thereby justifying a negative decision.

To assess whether algorithmic decision-making conflicts with reason rules, individuals need access to local decision rules that identify which minimal sets of feature values the model treats as sufficient grounds for its prediction. We believe that Anchors are the best available approach to approximate such locally sufficient conditions. An anchor identifies a set of feature constraints that, when satisfied, render the model’s prediction locally invariant with high probability in the surrounding neighborhood.

\begin{definition}[Anchor Explanation]
Let $x\in\mathcal{X}$, an \emph{anchor} is a rule $B$ (a set of feature conditions) such that, for points sampled from a local distribution $\pi_x$ that satisfy $B$, the decision remains stable with high \emph{precision}. The precision of an anchor $B$ is defined as $\text{prec}(B):=\mathbb{E}_{\pi_x\mid B}[\mathbf{1}_{\hat{d}(x)=\hat{d}(\pi_x)}]=1-\delta$ for some $\delta\in[0,1]$.
\end{definition}

The following theorem (proof in \Cref{app:proofs}) shows that Anchor explanations, given the Anchor indicates the opposite decision as the reason rule and both have sufficient overlap, can provide sufficient information to establish somewhere contestability. However, as the counterexample~\ref{ce:anchors} shows, they do not render decisions epistemically contestable.

\begin{restatable}{theorem}{retheoanchor}
\label{theo:anchor}
Let $i$ be an individual to whom the reason rule $A$ applies implying a positive (respectively negative) decision and the anchor $B$ applies, leading to a negative (respectively positive) decision with precision $\text{prec}(B)=1-\delta$ and $\mathbb{P}_{\pi_x}(B)>0$, and let $\pi$ be constrained by $\mathbb{P}_{\pi}(|\pi-\tilde{x}^{(i)}|>\epsilon)=0$. If $\mathbb{P}_{\pi_x}(A\cap B)>\delta$, then the decision is somewhere contestable.
\end{restatable}

\section{Evidence that allows for normative and strong epistemic contestability}
\label{sec:EpistemicReasons}

We have shown that conflicts between human intuitions and standard XAI explanations reveal errors somewhere; however, as argued, such errors are insufficient to justify overturning a decision. What kind of evidence, then, would enable individuals to assess and demonstrate that a decision is normatively or epistemically contestable? In the following, we examine three types of evidence that can be sufficient to warrant reversal: predictive multiplicity, incorrect feature values, and ignored overruling evidence.

\subsection{Predictive multiplicity reflects the chances for epistemic contestability}
\label{subsec:multiplicity}
Machine learning is, at its core, a tool of inductive inference. Given a dataset, e.g., containing creditworthy and non-creditworthy applicants, the goal is to learn a model that is able to distinguish loan applicants who will repay their loan from those who will not. Like any other tool of inductive inference, machine learning requires inductive assumptions to generalize to novel instances.

Modelers incorporate these inductive assumptions into the modeling pipeline through design choices that reflect their domain knowledge, always with the goal of best estimating the true underlying relationship $p$. Model classes, regularization techniques, and hyperparameters are selected based on this knowledge. In addition to such \emph{intentional choices}, model selection is also influenced by \emph{conventional choices}, such as the selection of a learning rate, or \emph{arbitrary choices}, such as the selection of random seeds \citep{ganesh2025systemizing,cavus2025role}. Even though inductive assumptions play a central role in guiding the model selection process, the iron rule in machine learning is that only models with strong performance on a holdout set are ultimately selected \citep{hardt2025emerging}.

As with all tools of inductive inference, this gives rise to a problem of \emph{contrastive underdetermination} \citep{laudan1990demystifying,sep-scientific-underdetermination}: multiple jointly incompatible models align equally well with both the observed data and the relevant domain knowledge, making it difficult to choose among them. In machine learning, this phenomenon has been called \emph{model multiplicity} or the \emph{Rashomon effect} \citep{breiman2001statistical,black2022model}. There always exists a multitude of prediction models that cannot be distinguished from an epistemic perspective: they result from the same model selection process and exhibit (almost) identical predictive performance.\footnote{Indeed, non-epistemic values \citep{sep-scientific-knowledge-social}, such as fairness or interpretability concerns, can guide model selection in such situations, an idea now widely discussed in the model multiplicity literature \citep{black2022model,xin2022exploring,rudin2024position,meyer2025perceptions,holtgen2025reconsidering}.} The set of such models has been called the \emph{Rashomon set}.

\begin{definition}[Rashomon model and set]
Let $S$ denote the model selection process, and let $H_{D,\delta}$ denote a random variable describing the intentional, conventional, and arbitrary inductive choices, which are informed by the dataset $D$ and the performance satisfaction level $\delta$. For a given model $\hat{p}:=S(H_{D,\delta},D)$, we call
\begin{itemize}
\item $\hat{q}$ a \emph{Rashomon model} relative to $\hat{p}$ if both result from the same model selection process applied to the same dataset and performance satisfaction level but differ in their inductive choices $H_{D,\delta}$; formally, $\hat{q}=S(H'_{D,\delta},D)$ for some $H'_{D,\delta}\neq H_{D,\delta}$ with positive density; and
\item $\text{RS}_{\hat{p}}:=\lbrace \hat{q}_1,\dots,\hat{q}_n\rbrace$ a \emph{Rashomon set} if it consists of $n$ randomly drawn Rashomon models, formally induced by i.i.d. sampled inductive choices $H_{D,\delta}^{(i)}\sim H_{D,\delta}$ with $\hat{q}_i=S(H_{D,\delta}^{(i)},D)$.
\end{itemize}
\end{definition}

Since all Rashomon models surpass the performance satisfaction level, they exhibit similar performance. Indeed, similar performance is often used as a defining criterion for Rashomon models \citep{fisher2019all,watson2023predictive,ganesh2025systemizing}. We argue, however, that arising from the same model-selection process provides a more appropriate definition, as it ensures not only comparable compatibility with the data across Rashomon models, but also consistency with researchers’ background knowledge as encoded in the inductive choices.

From a purely epistemic perspective, all Rashomon models must be regarded as on a par: they have equal claim to representing the ground truth dependence $p$. Despite this, they can differ substantially in their predictions for individual instances, a phenomenon known as \emph{predictive multiplicity} \citep{marx2020predictive}.\footnote{A variety of global measures for predictive multiplicity have been proposed \citep{watson2023predictive}.} Since these models have equal epistemic justification, so do their predictions. If some predictions are positive while others are negative, this provides a meaningful basis for contesting the decision, as argued by \citet{black2022model} and formalized in the following theorem. 

The theorem shows that, unless the decision maker has good reason to regard the selected model as superior, predictive multiplicity yields a lower bound on the probability that a decision is epistemically contestable. In this way, predictive multiplicity offers a principled route to establishing epistemic contestability.

\begin{restatable}{theorem}{retheorashomon}
Let $\hat{p}:=S(H_{D,\delta},D)$ denote the prediction model used in decision $\hat{d}$, obtained from a random draw from $H_{D,\delta}$, and let $\text{RS}_{\hat{p}}$ denote a Rashomon set. Assume there is no measurement error in the feature values, i.e. $\tilde{x}^{(i)}=x^{(i)}$. Then the probability that $\hat{p}$ is epistemically incorrect for an individual $i$ can be lower bounded by
\begin{equation*}
\scriptstyle{\mathbb{P}_{H_{D,\delta}}(d(x^{(i)})\neq\hat{d}(x^{(i)}))\geq \text{min}\lbrace \mathbb{P}_{H_{D,\delta}}(\hat{d}(x^{(i)})=1), \mathbb{P}_{H_{D,\delta}}(\hat{d}(x^{(i)})=0)\rbrace.}
\end{equation*}
We denote this lower bound by $\theta$, which can be estimated via the minimum of the proportions of positive and negative decisions in the Rashomon set, i.e.
\begin{equation*}
\scriptstyle{\hat{\theta}_{RS_{\hat{p}}}:=\text{min}\Big\lbrace\frac{1}{|RS_{\hat{p}}|}\sum\limits_{m\in RS_{\hat{p}}}\mathbf{1}_{\hat{d}_m(x^{(i)})=1},\frac{1}{|RS_{\hat{p}}|}\sum\limits_{m\in RS_{\hat{p}}}\mathbf{1}_{\hat{d}_m(x^{(i)})=0}\Big\rbrace.}
\end{equation*}
The estimator is asymptotically consistent, although it systematically underestimates $\theta$ in finite samples.\footnote{The finite-sample downward bias can be bounded; see \citet{clark1961greatest}.}
\end{restatable}

Predictive multiplicity, and thus the prospects for successful epistemic contestability, is high in two scenarios: (1) when all predictions are close to $\tau$ but some lie above and others below, corresponding to high aleatoric uncertainty; and (2) when predictions are widely spread, with some far above and others far below $\tau$, corresponding to high epistemic uncertainty \citep{hullermeier2021aleatoric}. 

\begin{policybox}
\textbf{Policy implication.} Individuals should be informed of the predictive multiplicity underlying their decision by reporting predictive multiplicity in their case. The respective Rashomon set should be randomly sampled from the space of plausible, high-performing models and constructed independently of the individual’s feature values.
\end{policybox}

\subsection{Wrong feature values can enable epistemic contestability}
\label{subsec:WrongFeatureValues}
Algorithmic decision-making crucially relies on the measured feature values of individuals $\tilde{x}^{(i)}$. In the context of lending, some of these features are provided to the credit institution by the applicant herself, such as age, marital status, or the requested credit amount. Others stem from information the bank has access to in its databases, such as the individual’s credit or transaction history. Finally, there is information that the bank actively gathers to assess the applicant’s trustworthiness, e.g., public records and court judgments \citep{dansana2024analyzing}.

During the processes of entering, transferring, or collecting data, errors can occur: the applicant's employment status may have changed; the credit history may contain redundant entries; or the collected public records may pertain to a different individual with the same name.\footnote{Similarly, feature values can be subject to uncertainty. For example, estimated monthly expenses or salary may vary substantially for self-employed or commission-based individuals.} Such errors are often not visible to the individual, as she typically only knows the features she has entered herself, and even those may contain unnoticed mistakes.

To make reliable decisions, a solid epistemic basis is indispensable. A decision can only be supported through evidence if that evidence is correct \citep{conee2004evidentialism}. Algorithmic decision-making thus inherits a standard familiar from theories of epistemic justification: if the informational premises on which an inference is based are defective, then the resulting belief or decision is epistemically unwarranted, even if the inferential tool itself is flawless \citep{wright2003some}.

If the measured feature values $\tilde{x}^{(i)}$ differ substantially from the true values $x^{(i)}$, this discrepancy can lead to corresponding errors in predictions and decisions \citep{sep-measurement-science}. To assess whether this is the case, individuals should be granted \emph{what-if access} to the model; that is, they should be allowed to test the model on a limited set of relevant input values.

\begin{definition}[What-if]
What-if provides individuals access to $\lbrace \hat{p}(x)\mid x\in\mathcal{X}_{rel}\rbrace$ for a limited set of relevant inputs $\mathcal{X}_{rel}$ specified by the individual subject to the decision.
\end{definition}

The following theorem shows that incorrect feature values, when corrected values would reverse the decision, enable individuals to epistemically contest the decision. While the premise of epistemic correctness may appear strong, denying it would require decision makers to question the validity of the very model they rely on.

\begin{restatable}{theorem}{retheoepcontest}
\label{theo:epicont}
Assume the decision model $\hat{d}$ is epistemically correct, i.e. $\hat{d}= d$. If individuals can prove that their true feature values are $x^{(i)}$ rather than $\tilde{x}^{(i)}$, and if $\hat{d}(x^{(i)})\neq\hat{d}(\tilde{x}^{(i)})$, then the decision is epistemically contestable.
\end{restatable}

\begin{policybox}
\textbf{Policy implication.} Individuals should be informed about their personal information used in the decision-making process. Additionally, they should be able to query the model using alternative inputs.
\end{policybox}

\subsection{Overruling evidence is a means towards normative contestability}
\label{subsec:Overruling}
Algorithmic decision-making reduces the epistemic basis of decisions to a fixed set of features to infer the unobserved property of interest, e.g., whether the loan will be repaid. While these features are carefully chosen, normatively incorrect decisions may still be made, even if the resulting decision closely approximates the epistemically correct one—even for optimal classifiers \citep{hastie2009elements}. 

Such a setting is highly static. By contrast, when humans make decisions, they flexibly incorporate additional information about other relevant feature values when these are presented to them. For example, if an individual is accused of cheating by her professor, she can provide further information to support her case, such as contemporaneous drafts, handwritten notes, or a version history demonstrating that the text was produced independently; this information is then taken into account in the professor’s decision. In such cases, a decision may remain epistemically correct relative to the feature set $X$, but it is no longer epistemically correct once the additional features are considered. We can transfer the idea underlying such \emph{overruling evidence} to the algorithmic case.

\begin{definition}[Overruling evidence]
We call features $E$ \emph{overruling evidence} for individual $i$ if $d_X(x^{(i)}) \neq d_{X,E}(x^{(i)},e^{(i)})$, where $d_X := \mathbf{1}_{\mathbb{P}(Y=1 \mid X) > \tau}$ and $d_{X,E} := \mathbf{1}_{\mathbb{P}(Y=1 \mid X, E) > \tau}$.
\end{definition}

Contesting a decision through overruling evidence differs fundamentally from other forms of contestation. Contestation based on XAI explanations, predictive multiplicity, or incorrect feature values all target the correctness of estimations or measurements within a fixed feature set $X$ and aim to establish epistemic contestability relative to $X$. Overruling evidence, by contrast, accepts that the decision is epistemically correct relative to $X$ but epistemically incorrect, and thus contestable, relative to the expanded feature set $X,E$. This invokes an epistemic meta-principle, namely the \emph{Principle of Total Evidence} \citep{carnap1947application}, according to which beliefs should always be based on all available evidence. This principle has strong normative support: it has been shown that new evidence can only increase, and never decrease, expected utility when the cost of information is negligible \citep{taroni2024principle,schurz2024principle}.

An analogous result is shown below for our setting, demonstrating that fewer or equal normatively incorrect decisions are made when additional evidence is taken into account. This implies that decisions based on a smaller evidential basis are generally less justified than decisions based on all available evidence, thereby supporting the idea that overruling evidence should overturn a decision.

\begin{restatable}{theorem}{retheofriendly}
\label{theo:friendly}
Assume that $X$ and $E$ describe two sets of features. Moreover, assume that the prediction models used are Bayes-optimal relative to $X$ and to $X,E$, respectively, i.e. $\hat{p}_X=p_X=\mathbb{P}(Y\mid X)$ and $\hat{p}_{X,E}=p_{X,E}=\mathbb{P}(Y\mid X,E)$. Consequently, $\hat{d}_X=d_X$ and $\hat{d}_{X,E}=d_{X,E}$. Then, the expected number of normatively contestable decisions relative to $X$ is greater than the expected number of normatively contestable decisions relative to $X,E$, i.e.
\begin{equation*}
\mathbb{E}_{Y,X,E}[d^*\neq \hat{d}_{X,E}]\leq \mathbb{E}_{Y,X}[d^*\neq \hat{d}_{X}].
\end{equation*}
\end{restatable}

Contesting a decision by means of overruling evidence is often the closest we can get to normative contestability, since the property of interest $Y$ is typically unobservable. For example, it is not possible at decision time to observe whether the loan will ultimately be repaid. However, even when overruling evidence sometimes exists, this is not necessarily always the case; consider a loan applicant for whom all observable evidence speaks against repayment, but who would later win the lottery and repay the loan. Rendering such a normatively incorrect decision epistemically contestable would require Laplace’s demon \citep{frigg2014laplace}.

Overruling evidence is appealing in theory but difficult to implement in practice. Decision makers typically do not have access to an approximation of $d_{X,E}$ for arbitrary evidence $E$. This leads to a reversal of roles \citep{brennan2019artificial}: individuals must explain why evidence $E$ changes their repayment probability sufficiently to warrant their desired decision. In practice, this often results in a negotiation process between the decision maker and the individual, in which both parties attempt to reach an agreement \citep{lipsky2010street}. Particularly when the deployed decision model exhibits low predictive performance\footnote{Low predictive performance can indicate that relevant features are missing, which, if available, would allow for better predictions.} or when uncertainty for a given individual is high, decision makers should have a strong incentive to take overruling evidence into account.

\begin{policybox}
\textbf{Policy implication.} Individuals should be allowed to submit additional evidence concerning features that were not considered in the original decision-making process. This is particularly critical when the deployed predictive model has low predictive performance or when the decision involves high uncertainty.
\end{policybox}

\section{Individuals may arguably have rights to this evidence according to the GDPR}
\label{sec:Legal}
So far, we have taken a purely \emph{normative perspective}, focusing on the question of which decisions \emph{should} be overturned from an epistemic standpoint. We argued that, to enable individuals to contest decisions, they should gain access to certain types of evidence, namely: information about predictive multiplicity, their feature values used, API access to query the model with a limited number of alternative inputs, and the ability to provide additional feature information. We now shift from this normative approach to a \emph{legal perspective}. Focusing on European legislation, we ask: Do citizens have legal rights to such evidence in the context of algorithmic contestability?

Contestability forms the “backbone” of Article 22 GDPR \citep{bayamliouglu2022right}. Specifically, the right to contest a decision is formally defined in Article 22(3) GDPR. This right applies only when a decision concerning the data subject is based solely on automated processing, as specified in Article 22(1) GDPR, which is generally prohibited unless the safeguards outlined in Article 22(2) GDPR are met. A recent ECJ judgment (C-634/21 SCHUFA Holding (Scoring), 7 December 2023, ECLI:EU:C:2023:957) ruled that, in the context of credit scoring, the concept of a 'decision' within the meaning of Article 22 GDPR is to be understood broadly.
Interestingly, despite the centrality of the right to contest within Article 22, it remains ``barely articulated'' in the regulation \citep{kaminski2021right}.\footnote{While the right to contest a decision is strictly defined in Article 22(3) GDPR, Article 22(1) GDPR can, in a broader sense, also be interpreted as providing a right to object to being subject to automated decision-making, thereby extending the concept beyond strict contestation.}


So far, Art. 22 GDPR together with the corresponding Recital 71 GDPR has primarily been interpreted by XAI researchers as promoting a right to explanations \citep{wachter2017counterfactual}, which in turn should empower data subjects to contest decisions \citep{selbst2018meaningful} (see also C-203/22 Dun \& Bradstreet, 27 February 2025, ECLI:EU:C:2025:117). However, Article 22 GDPR can be interpreted more broadly. First, it specifies also that individuals have the right to express their point of view, aligning with our recommendation that individuals should be able to provide additional feature information to support their case. Second, the right to an explanation is not limited to current XAI solutions but encompasses any type of evidence that allows individuals to better understand the decision-making process \citep{vredenburgh2022right}. This may include what-if access through an API for alternative input values. Similarly, this should arguably include the uncertainty in the decision for their particular case, which can be expressed through predictive multiplicity. This interpretation aligns with the decision maker’s obligation to use appropriate statistical procedures to minimize errors, as specified in Recital 71 GDPR:

\begin{quote}
    ``In order to ensure fair and transparent processing in respect of the data subject, [...] the controller should [...] implement technical and organisational measures appropriate to ensure, in particular, that factors which result in inaccuracies in personal data are corrected and the risk of errors is minimised [...]''.
\end{quote}

This paragraph also emphasizes the decision maker’s obligation to implement technical and organizational measures to correct inaccuracies in personal data. Providing individuals access to their specific data is arguably the most direct way to fulfill this obligation, as it allows them to intervene when inaccuracies occur. Given the importance of algorithmic contestation, decision-makers bear the responsibility of ensuring that all input data to the algorithm are accurate.

One might argue that in contexts like loan approvals, providing the types of information we recommend could impose significant demands on decision makers. Access to predictive multiplicity would require training multiple models with equal performance, all of which must meet stringent fairness and accountability standards. Providing individuals with their complete data and API access may expose proprietary company information \citep{sokol2019counterfactual}. Moreover, responding to all additional feature information submitted by individuals could be costly and burdensome.

However, without access to these types of information, individuals lack the necessary means to exercise their right to contestation, rendering the possibility that the right effectively becomes meaningless. In contexts where decision makers cannot provide sufficient support for contestation, it may be questionable whether the decision should be automated at all. Importantly, XAI explanations alone, as we have shown, are insufficient to meet the demands of the right to contest and should not be considered an adequate effort by decision makers to comply with the GDPR.

\section{Discussion}
In this paper, we conceptualized algorithmic contestability as the natural counterpart to traditional algorithmic recourse. Instead of recommending actions that change a decision, contestability concerns the provision of evidence that allows individuals to overturn a decision. We showed that the evidence provided by common XAI methods, together with human intuitions, allows only for the detection of errors somewhere and is insufficient to warrant decision reversal. Consequently, we moved beyond XAI and examined other sources of evidence that are, from a normative perspective, sufficient to overturn algorithmic decisions. We focused on predictive multiplicity, wrong feature values, and overruling evidence. For each, we grounded our analysis in epistemological theory and demonstrated that these forms of evidence indeed allow decisions to be overturned. We concluded by arguing that current legislation in the European Union, in particular the GDPR, already grants individuals access to these kinds of evidence that enable them to overturn decisions.

\textbf{Ethical reasons for contesting.} Our work focused on epistemic grounds for contestation; however, there can also be \emph{ethical reasons} for contestation. For instance, because the data may reflect societal biases or discrimination, using epistemically correct decisions as a standard becomes problematic, as these decisions are based on observational data influenced by such biases \citep{green2022escaping,barocas2023fairness}. Similarly, a model may rely on protected attributes in its decision-making that are predictive of the target but nevertheless discriminate against certain groups. Excluding such attributes is often insufficient to address this issue due to the existence of \emph{proxy} variables \citep{dwork2012fairness}. Since these ethical concerns are already widely discussed in the fairness community, we deliberately focused on the thus-far overlooked epistemic aspects of contestability. In practice, decisions can of course be contested on both ethical and epistemic grounds.

\textbf{Contesting when individuals and decision makers are adversarial.} For future work, we believe that incorporating a stakeholders’ perspective would be highly fruitful. In our setup, individuals can contest decisions by appealing to the ethical norms to which the decision maker commits. Thus, in our account, the goals of the involved parties are aligned: individuals who are likely to repay their loans aim to convince the decision maker that they will do so, while individuals who did not cheat aim to convince the decision maker that they acted honestly. Such situations are ideal, as contestability then consistently leads to better decisions \citep{cohen2023ai}. Many real-world settings, however, are adversarial, and the interests of individuals and decision makers are not aligned \citep{bordt2022post}. In such contexts, many open questions remain: Who provides the evidence for contestability, and who mediates between the individual and the decision maker?

\textbf{The limits of XAI explanations for contestability.} Our analysis was confined to a subset of human intuitions and XAI explanations. Clearly, there are further intuitive rules that humans adhere to; consider, for example, \emph{causal-relevance rules}, according to which only features that are causally relevant to the unobserved property should affect the decision \citep{sloman2015causality,ye2024spurious}. Likewise, other XAI explanations, such as Shapley interaction values \citep{muschalik2024shapiq}, global feature effect or importance measures \citep{molnar2020interpretable}, or approaches from mechanistic interpretability \citep{sharkey2025open}, could be analyzed with respect to their value for contestability. We believe that even if the analysis were extended along these lines, our core claim would remain: XAI explanations alone are insufficient to enable individuals to overturn decisions and must therefore be complemented by other forms of evidence.

\section*{Acknowledgments}
We thank Tom Sterkenburg and the participants of the MCMP colloquium for valuable feedback on the manuscript.

This research was supported by the German Research Foundation (DFG), project number 511917847; by the Cluster of Excellence “Machine Learning — New Perspectives for Science” (EXC 2064/1, project number 390727645); and by the Carl Zeiss Foundation through the CZS Center for AI and Law.

\section*{Generative AI Disclosure Statement}
All authors confirm that generative AI tools were used only for minor support in manuscript preparation. Specifically, we used ChatGPT (version 5) and Gemini (version 3) to assist with grammar and style editing, as well as with LaTeX drafting of boxes and tables. No text, claims, or references were generated by these tools. All authors take full responsibility for the originality, accuracy, and integrity of the manuscript. For coding, Claude Code (Opus 4.6) was used.

\bibliography{biblio}

@inproceedings{scholbeck2019sampling,
  title={Sampling, intervention, prediction, aggregation: a generalized framework for model-agnostic interpretations},
  author={Scholbeck, Christian A and Molnar, Christoph and Heumann, Christian and Bischl, Bernd and Casalicchio, Giuseppe},
  booktitle={Joint European conference on machine learning and knowledge discovery in databases},
  pages={205--216},
  year={2019},
  organization={Springer}
}

@article{frigg2014laplace,
  title={Laplace’s demon and the adventures of his apprentices},
  author={Frigg, Roman and Bradley, Seamus and Du, Hailiang and Smith, Leonard A},
  journal={Philosophy of Science},
  volume={81},
  number={1},
  pages={31--59},
  year={2014},
  publisher={Cambridge University Press}
}

@article{wright2003some,
  title={Some reflections on the acquisition of warrant by inference},
  author={Wright, Crispin},
  journal={New essays on semantic externalism and self-knowledge},
  pages={57--77},
  year={2003}
}

@InCollection{sep-measurement-science,
	author       =	{Tal, Eran},
	title        =	{{Measurement in Science}},
	booktitle    =	{The {Stanford} Encyclopedia of Philosophy},
	editor       =	{Edward N. Zalta},
	howpublished =	{\url{https://plato.stanford.edu/archives/fall2020/entries/measurement-science/}},
	year         =	{2020},
	edition      =	{{F}all 2020},
	publisher    =	{Metaphysics Research Lab, Stanford University}
}

@book{conee2004evidentialism,
  title={Evidentialism: Essays in epistemology},
  author={Conee, Earl and Feldman, Richard},
  year={2004},
  publisher={Clarendon Press}
}

@inproceedings{selbst2018meaningful,
  title={“Meaningful information” and the right to explanation},
  author={Selbst, Andrew and Powles, Julia},
  booktitle={conference on fairness, accountability and transparency},
  pages={48--48},
  year={2018},
  organization={PMLR}
}

@article{vredenburgh2022right,
  title={The right to explanation},
  author={Vredenburgh, Kate},
  journal={Journal of Political Philosophy},
  volume={30},
  number={2},
  pages={209--229},
  year={2022},
  publisher={Wiley Online Library}
}

@article{tversky1974judgment,
  title={Judgment under Uncertainty: Heuristics and Biases: Biases in judgments reveal some heuristics of thinking under uncertainty.},
  author={Tversky, Amos and Kahneman, Daniel},
  journal={science},
  volume={185},
  number={4157},
  pages={1124--1131},
  year={1974},
  publisher={American association for the advancement of science}
}

@book{gigerenzer2023intelligence,
  title={The intelligence of intuition},
  author={Gigerenzer, Gerd},
  year={2023},
  publisher={Cambridge University Press}
}

@article{kaminski2021right,
  title={The right to contest AI},
  author={Kaminski, Margot E and Urban, Jennifer M},
  journal={Columbia Law Review},
  volume={121},
  number={7},
  pages={1957--2048},
  year={2021},
  publisher={JSTOR}
}

@article{bayamliouglu2022right,
  title={The right to contest automated decisions under the General Data Protection Regulation: Beyond the so-called “right to explanation”},
  author={Bayaml{\i}o{\u{g}}lu, Emre},
  journal={Regulation \& Governance},
  volume={16},
  number={4},
  pages={1058--1078},
  year={2022},
  publisher={Wiley Online Library}
}

@article{dansana2024analyzing,
  title={Analyzing the impact of loan features on bank loan prediction using Random Forest algorithm},
  author={Dansana, Debabrata and Patro, S Gopal Krishna and Mishra, Brojo Kishore and Prasad, Vivek and Razak, Abdul and Wodajo, Anteneh Wogasso},
  journal={Engineering Reports},
  volume={6},
  number={2},
  pages={e12707},
  year={2024},
  publisher={Wiley Online Library}
}

@book{barocas2023fairness,
  title={Fairness and machine learning: Limitations and opportunities},
  author={Barocas, Solon and Hardt, Moritz and Narayanan, Arvind},
  year={2023},
  publisher={MIT Press}
}

@inproceedings{yurrita2023disentangling,
  title={Disentangling fairness perceptions in algorithmic decision-making: the effects of explanations, human oversight, and contestability},
  author={Yurrita, Mireia and Draws, Tim and Balayn, Agathe and Murray-Rust, Dave and Tintarev, Nava and Bozzon, Alessandro},
  booktitle={Proceedings of the 2023 CHI Conference on Human Factors in Computing Systems},
  pages={1--21},
  year={2023}
}

@InCollection{sep-scientific-underdetermination,
	author       =	{Stanford, Kyle},
	title        =	{{Underdetermination of Scientific Theory}},
	booktitle    =	{The {Stanford} Encyclopedia of Philosophy},
	editor       =	{Edward N. Zalta and Uri Nodelman},
	howpublished =	{\url{https://plato.stanford.edu/archives/sum2023/entries/scientific-underdetermination/}},
	year         =	{2023},
	edition      =	{{S}ummer 2023},
	publisher    =	{Metaphysics Research Lab, Stanford University}
}

@article{laudan1990demystifying,
  title={Demystifying underdetermination},
  author={Laudan, Larry},
  journal={Minnesota studies in the philosophy of science},
  volume={14},
  number={1990},
  pages={267--297},
  year={1990}
}

@article{anders2023using,
  title={Is using ChatGPT cheating, plagiarism, both, neither, or forward thinking?},
  author={Anders, Brent A},
  journal={Patterns},
  volume={4},
  number={3},
  year={2023},
  publisher={Elsevier}
}

@book{douglas2009science,
  title={Science, policy, and the value-free ideal},
  author={Douglas, Heather E},
  year={2009},
  publisher={University of Pittsburgh Press}
}

@article{henin2021framework,
  title={A framework to contest and justify algorithmic decisions},
  author={Henin, Cl{\'e}ment and Le M{\'e}tayer, Daniel},
  journal={AI and Ethics},
  volume={1},
  number={4},
  pages={463--476},
  year={2021},
  publisher={Springer}
}

@article{lyons2021conceptualising,
  title={Conceptualising contestability: Perspectives on contesting algorithmic decisions},
  author={Lyons, Henrietta and Velloso, Eduardo and Miller, Tim},
  journal={Proceedings of the ACM on Human-Computer Interaction},
  volume={5},
  number={CSCW1},
  pages={1--25},
  year={2021},
  publisher={ACM New York, NY, USA}
}

@article{mansi2025legally,
  title={Legally-Informed Explainable AI},
  author={Mansi, Gennie and Karusala, Naveena and Riedl, Mark},
  journal={arXiv preprint arXiv:2504.10708},
  year={2025}
}

@article{moreira2025explainable,
  title={Explainable AI Systems Must Be Contestable: Here's How to Make It Happen},
  author={Moreira, Catarina and Palatkina, Anna and Braca, Dacia and Walsh, Dylan M and Leihn, Peter J and Chen, Fang and Hubig, Nina C},
  journal={arXiv preprint arXiv:2506.01662},
  year={2025}
}

@article{lipton2018mythos,
  title={The mythos of model interpretability: In machine learning, the concept of interpretability is both important and slippery.},
  author={Lipton, Zachary C},
  journal={Queue},
  volume={16},
  number={3},
  pages={31--57},
  year={2018},
  publisher={ACM New York, NY, USA}
}

@inproceedings{pawelczyk2023privacy,
  title={On the privacy risks of algorithmic recourse},
  author={Pawelczyk, Martin and Lakkaraju, Himabindu and Neel, Seth},
  booktitle={International Conference on Artificial Intelligence and Statistics},
  pages={9680--9696},
  year={2023},
  organization={PMLR}
}

@inproceedings{von2022fairness,
  title={On the fairness of causal algorithmic recourse},
  author={Von K{\"u}gelgen, Julius and Karimi, Amir-Hossein and Bhatt, Umang and Valera, Isabel and Weller, Adrian and Sch{\"o}lkopf, Bernhard},
  booktitle={Proceedings of the AAAI conference on artificial intelligence},
  volume={36},
  number={9},
  pages={9584--9594},
  year={2022}
}

@article{upadhyay2021towards,
  title={Towards robust and reliable algorithmic recourse},
  author={Upadhyay, Sohini and Joshi, Shalmali and Lakkaraju, Himabindu},
  journal={Advances in Neural Information Processing Systems},
  volume={34},
  pages={16926--16937},
  year={2021}
}

@inproceedings{black2022model,
  title={Model multiplicity: Opportunities, concerns, and solutions},
  author={Black, Emily and Raghavan, Manish and Barocas, Solon},
  booktitle={Proceedings of the 2022 ACM conference on fairness, accountability, and transparency},
  pages={850--863},
  year={2022}
}

@book{lipsky2010street,
  title={Street-level bureaucracy: Dilemmas of the individual in public service},
  author={Lipsky, Michael},
  year={2010},
  publisher={Russell Sage Foundation}
}

@misc{cavus2025role,
      title={The Role of Hyperparameters in Predictive Multiplicity}, 
      author={Mustafa Cavus and Katarzyna Woźnica and Przemysław Biecek},
      year={2025},
      eprint={2503.13506},
      archivePrefix={arXiv},
      primaryClass={cs.LG},
      url={https://arxiv.org/abs/2503.13506}, 
}

@inproceedings{sokol2019counterfactual,
  title={Counterfactual explanations of machine learning predictions: opportunities and challenges for AI safety},
  author={Sokol, Kacper and Flach, Peter},
  booktitle={2019 AAAI Workshop on Artificial Intelligence Safety, SafeAI 2019},
  year={2019},
  organization={CEUR Workshop Proceedings}
}

@article{sloman2015causality,
  title={Causality in thought},
  author={Sloman, Steven A and Lagnado, David},
  journal={Annual review of psychology},
  volume={66},
  number={1},
  pages={223--247},
  year={2015},
  publisher={Annual Reviews}
}

@article{sharkey2025open,
  title={Open problems in mechanistic interpretability},
  author={Sharkey, Lee and Chughtai, Bilal and Batson, Joshua and Lindsey, Jack and Wu, Jeff and Bushnaq, Lucius and Goldowsky-Dill, Nicholas and Heimersheim, Stefan and Ortega, Alejandro and Bloom, Joseph and others},
  journal={arXiv preprint arXiv:2501.16496},
  year={2025}
}

@book{molnar2020interpretable,
  title={Interpretable machine learning},
  author={Molnar, Christoph},
  year={2020},
  publisher={Lulu.com}
}

@article{ye2024spurious,
  title={Spurious correlations in machine learning: A survey},
  author={Ye, Wenqian and Zheng, Guangtao and Cao, Xu and Ma, Yunsheng and Zhang, Aidong},
  journal={arXiv preprint arXiv:2402.12715},
  year={2024}
}

@article{green2022escaping,
  title={Escaping the impossibility of fairness: From formal to substantive algorithmic fairness},
  author={Green, Ben},
  journal={Philosophy \& Technology},
  volume={35},
  number={4},
  pages={90},
  year={2022},
  publisher={Springer}
}

@article{muschalik2024shapiq,
  title={shapiq: Shapley interactions for machine learning},
  author={Muschalik, Maximilian and Baniecki, Hubert and Fumagalli, Fabian and Kolpaczki, Patrick and Hammer, Barbara and H{\"u}llermeier, Eyke},
  journal={Advances in Neural Information Processing Systems},
  volume={37},
  pages={130324--130357},
  year={2024}
}

@inproceedings{dwork2012fairness,
  title={Fairness through awareness},
  author={Dwork, Cynthia and Hardt, Moritz and Pitassi, Toniann and Reingold, Omer and Zemel, Richard},
  booktitle={Proceedings of the 3rd innovations in theoretical computer science conference},
  pages={214--226},
  year={2012}
}

@inproceedings{bordt2022post,
  title={Post-hoc explanations fail to achieve their purpose in adversarial contexts},
  author={Bordt, Sebastian and Finck, Mich{\`e}le and Raidl, Eric and Von Luxburg, Ulrike},
  booktitle={Proceedings of the 2022 ACM Conference on Fairness, Accountability, and Transparency},
  pages={891--905},
  year={2022}
}

@article{breiman2001statistical,
  title={Statistical modeling: The two cultures (with comments and a rejoinder by the author)},
  author={Breiman, Leo},
  journal={Statistical science},
  volume={16},
  number={3},
  pages={199--231},
  year={2001},
  publisher={Institute of Mathematical Statistics}
}

@article{hullermeier2021aleatoric,
  title={Aleatoric and epistemic uncertainty in machine learning: An introduction to concepts and methods},
  author={H{\"u}llermeier, Eyke and Waegeman, Willem},
  journal={Machine learning},
  volume={110},
  number={3},
  pages={457--506},
  year={2021},
  publisher={Springer}
}

@inproceedings{watson2023predictive,
  title={Predictive multiplicity in probabilistic classification},
  author={Watson-Daniels, Jamelle and Parkes, David C and Ustun, Berk},
  booktitle={Proceedings of the AAAI Conference on Artificial Intelligence},
  volume={37},
  number={9},
  pages={10306--10314},
  year={2023}
}

@inproceedings{marx2020predictive,
  title={Predictive multiplicity in classification},
  author={Marx, Charles and Calmon, Flavio and Ustun, Berk},
  booktitle={International conference on machine learning},
  pages={6765--6774},
  year={2020},
  organization={PMLR}
}

@article{xin2022exploring,
  title={Exploring the whole rashomon set of sparse decision trees},
  author={Xin, Rui and Zhong, Chudi and Chen, Zhi and Takagi, Takuya and Seltzer, Margo and Rudin, Cynthia},
  journal={Advances in neural information processing systems},
  volume={35},
  pages={14071--14084},
  year={2022}
}

@inproceedings{rudin2024position,
  title={Position: amazing things come from having many good models},
  author={Rudin, Cynthia and Zhong, Chudi and Semenova, Lesia and Seltzer, Margo and Parr, Ronald and Liu, Jiachang and Katta, Srikar and Donnelly, Jon and Chen, Harry and Boner, Zachery},
  booktitle={Proceedings of the 41st International Conference on Machine Learning},
  pages={42783--42795},
  year={2024}
}

@article{fisher2019all,
  title={All models are wrong, but many are useful: Learning a variable's importance by studying an entire class of prediction models simultaneously},
  author={Fisher, Aaron and Rudin, Cynthia and Dominici, Francesca},
  journal={Journal of Machine Learning Research},
  volume={20},
  number={177},
  pages={1--81},
  year={2019}
}

@inproceedings{holtgen2025reconsidering,
  title={Reconsidering Fairness Through Unawareness From the Perspective of Model Multiplicity},
  author={H{\"o}ltgen, Benedikt and Oliver, Nuria},
  booktitle={Proceedings of the 5th ACM Conference on Equity and Access in Algorithms, Mechanisms, and Optimization},
  pages={241--255},
  year={2025}
}

@article{snedegar2018reasons,
  title={Reasons for and reasons against},
  author={Snedegar, Justin},
  journal={Philosophical Studies},
  volume={175},
  number={3},
  pages={725--743},
  year={2018},
  publisher={Springer}
}

@inproceedings{albini2022counterfactual,
  title={Counterfactual shapley additive explanations},
  author={Albini, Emanuele and Long, Jason and Dervovic, Danial and Magazzeni, Daniele},
  booktitle={Proceedings of the 2022 ACM conference on fairness, accountability, and transparency},
  pages={1054--1070},
  year={2022}
}

@inproceedings{ribeiro2018anchors,
  title={Anchors: High-precision model-agnostic explanations},
  author={Ribeiro, Marco Tulio and Singh, Sameer and Guestrin, Carlos},
  booktitle={Proceedings of the AAAI conference on artificial intelligence},
  volume={32},
  number={1},
  year={2018}
}

@article{wang2024gradient,
  title={Gradient based feature attribution in explainable AI: A technical review},
  author={Wang, Yongjie and Zhang, Tong and Guo, Xu and Shen, Zhiqi},
  journal={arXiv preprint arXiv:2403.10415},
  year={2024}
}

@article{goldstein2015peeking,
  title={Peeking inside the black box: Visualizing statistical learning with plots of individual conditional expectation},
  author={Goldstein, Alex and Kapelner, Adam and Bleich, Justin and Pitkin, Emil},
  journal={Journal of Computational and Graphical Statistics},
  volume={24},
  number={1},
  pages={44--65},
  year={2015},
  publisher={Taylor \& Francis}
}

@incollection{birnbaum2019violations,
  title={Violations of monotonicity in judgment and decision making},
  author={Birnbaum, Michael H},
  booktitle={Choice, decision, and measurement},
  pages={73--100},
  year={2019},
  publisher={Routledge}
}

@article{freiesleben2022intriguing,
  title={The intriguing relation between counterfactual explanations and adversarial examples},
  author={Freiesleben, Timo},
  journal={Minds and Machines},
  volume={32},
  number={1},
  pages={77--109},
  year={2022},
  publisher={Springer}
}

@book{kolodner2014case,
  title={Case-based reasoning},
  author={Kolodner, Janet},
  year={2014},
  publisher={Morgan Kaufmann}
}

@article{gigerenzer2011heuristic,
  title={Heuristic decision making},
  author={Gigerenzer, Gerd and Gaissmaier, Wolfgang},
  journal={Annual review of psychology},
  volume={62},
  number={2011},
  pages={451--482},
  year={2011},
  publisher={Annual Reviews}
}

@article{schmude2025two,
  title={" Two Means to an End Goal": Connecting Explainability and Contestability in the Regulation of Public Sector AI},
  author={Schmude, Timoth{\'e}e and Yurrita, Mireia and Alfrink, Kars and Goff, Thomas Le and Viard, Tiphaine},
  journal={arXiv preprint arXiv:2504.18236},
  year={2025}
}

@article{konig2025performative,
  title={Performative validity of recourse explanations},
  author={K{\"o}nig, Gunnar and Fokkema, Hidde and Freiesleben, Timo and Mendler-D{\"u}nner, Celestine and von Luxburg, Ulrike},
  journal={Advances in Neural Information Processing Systems},
  volume={39},
  year={2025}
}

@article{maxwell2023meaningful,
  title={Meaningful XAI based on user-centric design methodology: Combining legal and human-computer interaction (HCI) approaches to achieve meaningful algorithmic explainability},
  author={Maxwell, Winston and Dumas, Bruno},
  journal={Available at SSRN 4520754},
  year={2023}
}

@article{douglas2000inductive,
  title={Inductive risk and values in science},
  author={Douglas, Heather},
  journal={Philosophy of science},
  volume={67},
  number={4},
  pages={559--579},
  year={2000},
  publisher={Cambridge University Press}
}

@inproceedings{aler2020contestable,
  title={Contestable black boxes},
  author={Aler Tubella, Andrea and Theodorou, Andreas and Dignum, Virginia and Michael, Loizos},
  booktitle={International Joint Conference on Rules and Reasoning},
  pages={159--167},
  year={2020},
  organization={Springer}
}

@misc{AI_Ethics_Framework_2019,
  author       = {{Department of Industry, Innovation and Science}},
  title        = {{Artificial Intelligence: Australia's Ethics Framework}},
  organization = {{Law Council of Australia}},
  year         = {2019},
  month        = {jun},
  url          = {https://lawcouncil.au/publicassets/afebc52d-afa6-e911-93fe-005056be13b5/3639\%20-\%20AI\%20ethics.pdf}
}

@article{cohen2023ai,
  title={How AI can learn from the law: putting humans in the loop only on appeal},
  author={Cohen, I Glenn and Babic, Boris and Gerke, Sara and Xia, Qiong and Evgeniou, Theodoros and Wertenbroch, Klaus},
  journal={npj Digital Medicine},
  volume={6},
  number={1},
  pages={160},
  year={2023},
  publisher={Nature Publishing Group UK London}
}

@inproceedings{vaccaro2019contestability,
  title={Contestability in algorithmic systems},
  author={Vaccaro, Kristen and Karahalios, Karrie and Mulligan, Deirdre K and Kluttz, Daniel and Hirsch, Tad},
  booktitle={Companion Publication of the 2019 Conference on Computer Supported Cooperative Work and Social Computing},
  pages={523--527},
  year={2019}
}

@article{yurrita2025identifying,
  title={Identifying Algorithmic Decision Subjects' Needs for Meaningful Contestability},
  author={Yurrita, Mireia and Verma, Himanshu and Balayn, Agathe and Alfrink, Kars and Gadiraju, Ujwal and Bozzon, Alessandro},
  journal={Proceedings of the ACM on Human-Computer Interaction},
  volume={9},
  number={7},
  pages={1--29},
  year={2025},
  publisher={ACM New York, NY, USA}
}

@inproceedings{hirsch2017designing,
  title={Designing contestability: Interaction design, machine learning, and mental health},
  author={Hirsch, Tad and Merced, Kritzia and Narayanan, Shrikanth and Imel, Zac E and Atkins, David C},
  booktitle={Proceedings of the 2017 Conference on Designing Interactive Systems},
  pages={95--99},
  year={2017}
}

@article{sarra2020put,
  title={Put dialectics into the machine: protection against automatic-decision-making through a deeper understanding of contestability by design},
  author={Sarra, Claudio},
  journal={Global Jurist},
  volume={20},
  number={3},
  pages={20200003},
  year={2020},
  publisher={De Gruyter}
}

@article{alfrink2023contestable,
  title={Contestable AI by design: Towards a framework},
  author={Alfrink, Kars and Keller, Ianus and Kortuem, Gerd and Doorn, Neelke},
  journal={Minds and Machines},
  volume={33},
  number={4},
  pages={613--639},
  year={2023},
  publisher={Springer}
}

@inproceedings{konig2023improvement,
  title={Improvement-focused causal recourse (ICR)},
  author={K{\"o}nig, Gunnar and Freiesleben, Timo and Grosse-Wentrup, Moritz},
  booktitle={Proceedings of the AAAI Conference on Artificial Intelligence},
  volume={37},
  number={10},
  pages={11847--11855},
  year={2023}
}

@inproceedings{ustun2019actionable,
  title={Actionable recourse in linear classification},
  author={Ustun, Berk and Spangher, Alexander and Liu, Yang},
  booktitle={Proceedings of the conference on fairness, accountability, and transparency},
  pages={10--19},
  year={2019}
}

@article{karimi2022survey,
  title={A survey of algorithmic recourse: contrastive explanations and consequential recommendations},
  author={Karimi, Amir-Hossein and Barthe, Gilles and Sch{\"o}lkopf, Bernhard and Valera, Isabel},
  journal={ACM Computing Surveys},
  volume={55},
  number={5},
  pages={1--29},
  year={2022},
  publisher={ACM New York, NY}
}

@inproceedings{venkatasubramanian2020philosophical,
  title={The philosophical basis of algorithmic recourse},
  author={Venkatasubramanian, Suresh and Alfano, Mark},
  booktitle={Proceedings of the 2020 conference on fairness, accountability, and transparency},
  pages={284--293},
  year={2020}
}

@article{brennan2019artificial,
  title={Artificial intelligence and role-reversible judgment},
  author={Brennan-Marquez, Kiel and Henderson, Stephen E},
  journal={J. Crim. L. \& Criminology},
  volume={109},
  pages={137},
  year={2019},
  publisher={HeinOnline}
}

@book{van2000asymptotic,
  title={Asymptotic statistics},
  author={Van der Vaart, Aad W},
  volume={3},
  year={2000},
  publisher={Cambridge university press}
}

@article{clark1961greatest,
  title={The greatest of a finite set of random variables},
  author={Clark, Charles E},
  journal={Operations Research},
  volume={9},
  number={2},
  pages={145--162},
  year={1961},
  publisher={INFORMS}
}

@article{selbst2018intuitive,
  title={The intuitive appeal of explainable machines},
  author={Selbst, Andrew D and Barocas, Solon},
  journal={Fordham L. Rev.},
  volume={87},
  pages={1085},
  year={2018},
  publisher={HeinOnline}
}

@article{mulligan2019shaping,
  title={Shaping our tools: Contestability as a means to promote responsible algorithmic decision making in the professions},
  author={Mulligan, Deirdre K and Kluttz, Daniel and Kohli, Nitin},
  journal={Available at SSRN 3311894},
  year={2019}
}

@book{hastie2009elements,
  title={The elements of statistical learning},
  author={Hastie, Trevor and Tibshirani, Robert and Friedman, Jerome and others},
  year={2009},
  publisher={Springer}
}

@article{adadi2018peeking,
  title={Peeking inside the black-box: a survey on explainable artificial intelligence (XAI)},
  author={Adadi, Amina and Berrada, Mohammed},
  journal={IEEE access},
  volume={6},
  pages={52138--52160},
  year={2018},
  publisher={IEEE}
}

@misc{hardt2025emerging,
  author = {Moritz Hardt},
  title = {The Emerging Science of Machine Learning Benchmarks},
  year = {2025},
  howpublished = {Online at \url{https://mlbenchmarks.org}},
  note = {Manuscript}
}

@inproceedings{meyer2025perceptions,
  title={Perceptions of the Fairness Impacts of Multiplicity in Machine Learning},
  author={Meyer, Anna P and Kim, Yea-Seul and D'Antoni, Loris and Albarghouthi, Aws},
  booktitle={Proceedings of the 2025 CHI Conference on Human Factors in Computing Systems},
  pages={1--15},
  year={2025}
}

@InCollection{sep-scientific-knowledge-social,
	author       =	{Longino, Helen},
	title        =	{{The Social Dimensions of Scientific Knowledge}},
	booktitle    =	{The {Stanford} Encyclopedia of Philosophy},
	editor       =	{Edward N. Zalta and Uri Nodelman},
	howpublished =	{\url{https://plato.stanford.edu/archives/spr2025/entries/scientific-knowledge-social/}},
	year         =	{2025},
	edition      =	{{S}pring 2025},
	publisher    =	{Metaphysics Research Lab, Stanford University}
}

@inproceedings{dandl2020multi,
  title={Multi-objective counterfactual explanations},
  author={Dandl, Susanne and Molnar, Christoph and Binder, Martin and Bischl, Bernd},
  booktitle={International conference on parallel problem solving from nature},
  pages={448--469},
  year={2020},
  organization={Springer}
}

@article{raz1975reasons,
  title={Reasons for action, decisions and norms},
  author={Raz, Joseph},
  journal={Mind},
  volume={84},
  number={336},
  pages={481--499},
  year={1975},
  publisher={JSTOR}
}

@article{schurz2024principle,
  title={The Principle of Total Evidence: Justification and Political Significance},
  author={Schurz, Gerhard},
  journal={Acta Analytica},
  volume={39},
  number={4},
  pages={677--692},
  year={2024},
  publisher={Springer}
}

@article{taroni2024principle,
  title={The principle of total evidence reprised},
  author={Taroni, Franco and Aitken, Colin and Bozza, Silvia and Juchli, Patrick},
  journal={Law, Probability and Risk},
  volume={23},
  number={1},
  pages={mgae011},
  year={2024},
  publisher={Oxford University Press}
}

@article{carnap1947application,
  title={On the application of inductive logic},
  author={Carnap, Rudolf},
  journal={Philosophy and phenomenological research},
  volume={8},
  number={1},
  pages={133--148},
  year={1947},
  publisher={JSTOR}
}

@inproceedings{ganesh2025systemizing,
  title={Systemizing Multiplicity: The Curious Case of Arbitrariness in Machine Learning},
  author={Ganesh, Prakhar and Taik, Afaf and Farnadi, Golnoosh},
  booktitle={Proceedings of the AAAI/ACM Conference on AI, Ethics, and Society},
  volume={8},
  number={2},
  pages={1032--1048},
  year={2025}
}

@inproceedings{ribeiro2016should,
  title={" Why should I trust you?" Explaining the predictions of any classifier},
  author={Ribeiro, Marco Tulio and Singh, Sameer and Guestrin, Carlos},
  booktitle={Proceedings of the 22nd ACM SIGKDD international conference on knowledge discovery and data mining},
  pages={1135--1144},
  year={2016}
}

@article{kusner2017counterfactual,
  title={Counterfactual fairness},
  author={Kusner, Matt J and Loftus, Joshua and Russell, Chris and Silva, Ricardo},
  journal={Advances in neural information processing systems},
  volume={30},
  year={2017}
}

@article{wachter2017counterfactual,
  title={Counterfactual explanations without opening the black box: Automated decisions and the GDPR},
  author={Wachter, Sandra and Mittelstadt, Brent and Russell, Chris},
  journal={Harv. JL \& Tech.},
  volume={31},
  pages={841},
  year={2017},
  publisher={HeinOnline}
}

@Article{openml,
    title = {{OpenML}: Networked Science in Machine Learning},
    author = {Joaquin Vanschoren and Jan {N. van Rijn} and Bernd Bischl
      and Luis Torgo},
    journal = {SIGKDD Explorations},
    volume = {15},
    year = {2013},
    pages = {49--60},
    doi = {10.1145/2641190.2641198},
    url = {https://dx.doi.org/10.1145/2641190.2641198},
  }

@misc{south_german_credit_573,
  title        = {{South German Credit}},
  year         = {2020},
  howpublished = {UCI Machine Learning Repository},
  note         = {{DOI}: https://doi.org/10.24432/C5QG88}
}

@article{scikit-learn,
  title={Scikit-learn: Machine Learning in {P}ython},
  author={Pedregosa, F. and Varoquaux, G. and Gramfort, A. and Michel, V.
          and Thirion, B. and Grisel, O. and Blondel, M. and Prettenhofer, P.
          and Weiss, R. and Dubourg, V. and Vanderplas, J. and Passos, A. and
          Cournapeau, D. and Brucher, M. and Perrot, M. and Duchesnay, E.},
  journal={Journal of Machine Learning Research},
  volume={12},
  pages={2825--2830},
  year={2011}
}

\newpage
\onecolumn
\appendix

\section{Philosophical Glossary}
\label{app:philGloss}
\begin{table}[ht]
\centering
\label{tab:glossary}
\arrayrulecolor{gray} 
\begin{tabularx}{\textwidth}{@{} l X @{}}
    \toprule[1.5pt]
    \textbf{Term} & \textbf{Conceptual Definition in the Context of Contestation} \\
    \midrule[1pt]
    \textbf{Epistemic} & Relating to knowledge, justification, and the degree of rational belief.  \\
    \addlinespace
    \textbf{Intuitions} & Rules of thumb that humans rely on in belief formation and action, for example, monotonicity or continuity rules. \\
    \addlinespace
    \textbf{Ethical} & Refers to the specific maxims or principles to which a decision-maker commits in guiding their actions and policies. The term does not imply moral justification. \\
    \addlinespace
    \textbf{Inductive Risk} & The risk of error involved in accepting or rejecting a hypothesis. In lending, this is the risk of misclassifying a ``good'' loan applicant as ``bad'' based on incomplete evidence. \\
    \addlinespace
    \textbf{Normative} & Relating to an evaluative standard or ``ought'' statement. \\
    \addlinespace
    \textbf{Principle of Total Evidence} & An epistemic meta-principle stating that beliefs and decisions should always be based on all available evidence. \\
    \addlinespace
    \textbf{Underdetermination} & A state where multiple jointly incompatible models align equally well with both the observed data and relevant domain knowledge, making it difficult to choose among them. \\
    \bottomrule[1.5pt]
\end{tabularx}
\end{table}

\section{Proofs}
\label{app:proofs}
\begin{theorem}
\label{theo:Epistemic-weakly}
Assume $\tilde{x}^{(i)}=x^{(i)}$. Then, two logical relations can be shown:
\begin{enumerate}[i)]
    \item If the decision is epistemically contestable, then the model is somewhere contestable.
    \item If the decision is somewhere contestable, then the model is somewhere inaccurate.
\end{enumerate} 
\end{theorem}

\begin{proof}
    We start with i), assume that the decision about individual $i$ is epistemically contestable, then $\hat{d}(\tilde{x}^{(i)})\neq d(x^{(i)})$. Since $\tilde{x}^{(i)}=x^{(i)}$, we know that $\hat{d}(\tilde{x}^{(i)})\neq d(\tilde{x}^{(i)})$. Now because $\|\tilde{x}^{(i)}-\tilde{x}^{(i)}\|=0$, we can infer that there exists an $x\in \mathcal{B}_{\epsilon}(\tilde{x}^{(i)}): \hat{d}(x)\neq d(x)$, i.e. the decision is weakly epistemically contestable.

    Consider ii), assume that the decision about individual $i$ is somewhere contestable. This means there exists an $x\in \mathcal{B}_{\epsilon}(\tilde{x}^{(i)}): \hat{d}(x)\neq d(x)$. Consider a particular $x'\in\mathcal{B}_{\epsilon}(\tilde{x}^{(i)}): \hat{d}(x')\neq d(x')$. Then, either $\hat{p}(x') > \tau$ and $p(x') \leq \tau$ or $\hat{p}(x') \leq \tau$ and $p(x') > \tau$. In both cases, this implies $p(x')\neq\hat{p}(x')$ and because $x'\in\mathcal{B}_{\epsilon}(\tilde{x}^{(i)})$, together, we can infer somewhere inaccuracy.
\end{proof}

\retheocount*

\begin{proof}
    Let $i$ and $i_c$ be individuals to whom the continuity rule applies, then $d(\tilde{x}^{(i)})=d(\tilde{x}^{(i_c)})$.  Since $\tilde{x}^{(i_c)}$ is the counterfactual to $\tilde{x}^{(i)}$, we know that $\hat{d}(\tilde{x}^{(i)})=0$ and $\hat{d}(\tilde{x}^{(i_c)})=1$ or $\hat{d}(\tilde{x}^{(i)})=1$ and $\hat{d}(\tilde{x}^{(i_c)})=0$. 

    There now can be two cases: (I) $\hat{d}(\tilde{x}^{(i)})\neq d(\tilde{x}^{(i)})$; (II) $\hat{d}(\tilde{x}^{(i)})= d(\tilde{x}^{(i)})$.

    If (I), then the decision is epistemically contestable and hence also somewhere contestable as shown in Theorem \ref{theo:Epistemic-weakly}.
    
    If (II), then using $\hat{d}(\tilde{x}^{(i)})= d(\tilde{x}^{(i)})$ together with the continuity rule, we can infer that $\hat{d}(\tilde{x}^{(i)})=d(\tilde{x}^{(i_c)})$. Since we know that $\tilde{x}^{(i_c)}$ is a counterfactual receiving a different decision than $\tilde{x}^{(i)}$, we can infer that $\hat{d}(\tilde{x}^{(i_c)})\neq d(\tilde{x}^{(i_c)})$. Since we additionally know that their distance is smaller than $\epsilon$, we can infer that the decision is somewhere contestable.
\end{proof}

\retheolime*
\begin{proof}
    We will constrain our proof to the case where the monotonicity rule applies positively to feature $k$ and that the LIME explainer yields $w_k<0$, the reverse case works analogously.

    Our proof strategy is to show this theorem by contraposition, i.e. we show that if the model is not somewhere inaccurate, then the local faithfulness must be greater than or equal to $\mathbb{V}_{\pi_k}[w_k \pi_k]$. 
    
    Assume that the model is not somewhere inaccurate. Then, for all $x\in\mathcal{X}$ with $\|x-\tilde{x}^{(i)}\|<\epsilon$ holds $\hat{p}(x)=p(x)$. Since $\mathbb{P}_{\pi}(|\pi-\tilde{x}^{(i)}|>\epsilon)=0$, this implies $\mathbb{P}_{\pi}(\hat{p}(\pi)=p(\pi))=1$. 
    
    We denote $\pi_{-k}:=(\pi_1,\dots,\pi_{k-1},\pi_{k+1},\dots,\pi_r)$ and obtain 
    \begin{align*}
        \text{LF}(g,\hat{p})&=\mathbb{E}_{\pi}[(g(\pi)-\hat{p}(\pi))^2]\\
        &\overset{\text{(i)}}=\mathbb{E}_{\pi}[(g(\pi)-p(\pi))^2]\\
        &\overset{\text{(ii)}}=\mathbb{E}_{\pi_k}[\mathbb{E}_{\pi_{-k}\mid \pi_k}[(g(\pi_{-k},\pi_k)-p(\pi_{-k},\pi_k))^2\mid \pi_k]]\\
        &\overset{\text{(iii)}}=\mathbb{E}_{\pi_k}[\mathbb{E}_{\pi_{-k}}[(g(\pi_{-k},\pi_k)-p(\pi_{-k},\pi_k))^2]]\\
        &\overset{\text{(iv)}}\geq\mathbb{E}_{\pi_k}[\mathbb{E}_{\pi_{-k}}[g(\pi_{-k},\pi_k)-p(\pi_{-k},\pi_k)]^2]\\
        &\overset{\text{(v)}}=\mathbb{E}_{\pi_k}[(w_k\pi_k-\mathbb{E}_{\pi_{-k}}[p(\pi_{-k},\pi_k)-\sum\limits_{\substack{i=1\\ i\neq k}}^rw_i\pi_i])^2]\\
        &\overset{\text{(vi)}}\geq \underset{p'\text{ positive monotonic in }k}{\text{min}}\;\mathbb{E}_{\pi_k}[(w_k\pi_k-\mathbb{E}_{\pi_{-k}}[p(\pi_{-k},\pi_k)-\sum\limits_{\substack{i=1\\ i\neq k}}^rw_i\pi_i])^2]\\
        &\overset{\text{(vii)}}\geq \underset{q'\text{ positive monotonic}}{\text{min}}\;\mathbb{E}_{\pi_k}[(w_k\pi_k-q'(\pi_k))^2]\\
        &\overset{\text{(viii)}}= \underset{c\in \mathbb{R}}{\text{min}}\;\mathbb{E}_{\pi_k}[(w_k \pi_k-c)^2]\\
        &\overset{\text{(ix)}}= \mathbb{E}_{\pi_k}[(w_k \pi_k-\mathbb{E}[w_k \pi_k])^2]\\
        &= \mathbb{V}_{\pi_k}[w_k \pi_k]
    \end{align*}
     In the derivation, we used the following: 
     \begin{enumerate}[(i)]
         \item $\mathbb{P}_{\pi}(\hat{p}(\pi)=p(\pi))=1$, which we obtained through the somewhere inaccuracy assumption. 
         \item The tower rule for expectations.
         \item The independence of $\pi_k$ and $\pi_{-k}$.
         \item Jensen's inequality.
         \item Linearity of expectation value.
         \item By assumption, $p$ is positive monotonic in $k$. The term is therefore greater than or equal to the minimum over all functions that are positive monotonic in $k$.
         \item We can define the function $q(\pi_k):=\mathbb{E}_{\pi_{-k}}[(p(\pi_{-k},\pi_k)-\sum\limits_{\substack{i=1\\ i\neq k}}^rw_i\pi_i]$. We find that if $p$ is positive monotonic in $k$, then also $q$ is monotonically increasing in $k$. Let $z_1\leq z_2$, then 
         \begin{align*}
         p(\pi_{-k},z_1)&\leq p(\pi_{-k},z_2) \Rightarrow p(\pi_{-k},z_1)-\sum\limits_{\substack{i=1\\ i\neq k}}^rw_i\pi_i\leq p(\pi_{-k},z_2)-\sum\limits_{\substack{i=1\\ i\neq k}}^rw_i\pi_i\\ 
         &\Rightarrow\mathbb{E}_{\pi_{-k}}[p(\pi_{-k},z_1)-\sum\limits_{\substack{i=1\\ i\neq k}}^rw_i\pi_i]\leq \mathbb{E}_{\pi_{-k}}[p(\pi_{-k},z_2)-\sum\limits_{\substack{i=1\\ i\neq k}}^rw_i\pi_i]\\
         &\Leftrightarrow q(z_1)\leq q(z_2).
         \end{align*} 
         \item The smallest approximating positive monotonic unary function to a negative monotonic unary function is a constant. 
         \item For an arbitrary function, the constant with the smallest distance according to the mean squared error is the expected value of the function \citep{hastie2009elements}.
     \end{enumerate}
     
\end{proof}

\retheoanchor*

\begin{proof}We prove this for the case when the reasons rule $A$ implies a positive decision and the anchor implies a negative decision, the reverse case works analogously.

Since $B$ is an anchor with precision $1-\delta$, we know that $\mathbb{P}_{\pi_x}(\hat{d}(\pi_x)=0\mid B)=1-\delta$. By the Kolmogorov axioms, we can infer $\mathbb{P}_{\pi_x}(\hat{d}(\pi_x)=1\mid B)=\delta$. Due to our reasons rule, we know that $\mathbb{P}_{\pi_x}(d(\pi_x)=1\mid A)=1$.

We now prove our theorem by contraposition and assume that the model is not somewhere contestable. Then, we can infer
    \begin{align*}
        \delta&=\mathbb{P}_{\pi_x}(\hat{d}(\pi_x)=1\mid B)=\frac{\mathbb{P}_{\pi_x}(\hat{d}(\pi_x)=1\cap B)}{\mathbb{P}_{\pi_x}(B)}\\
        &\overset{\text{(i)}}\geq \frac{\mathbb{P}_{\pi_x}(\hat{d}(\pi_x)=1\cap B\cap A)}{\mathbb{P}_{\pi_x}(B)}\\
        &\overset{\text{(ii)}}{=}\frac{\mathbb{P}_{\pi_x}(d(\pi_x)=1\cap B\cap A)}{\mathbb{P}_{\pi_x}(B)}\\
        &\overset{\text{(iii)}}{=}\frac{\mathbb{P}_{\pi_x}( B\cap A)}{\mathbb{P}_{\pi_x}(B)}\\
        &\overset{\text{(iv)}}\geq \mathbb{P}_{\pi_x}( B\cap A).
    \end{align*}
    In the derivation we used the following:
    \begin{enumerate}[(i)]
        \item Adding more sets in the intersection in the probability term of the numerator only reduces the probability.
        \item  $\forall x\in \mathcal{B}_{\epsilon}(x^{(i)}): d(x)=\hat{d}(x)$ by the assumption that the model is not somewhere contestable and we know that $\pi$ is constrained by $\mathbb{P}_{\pi}(|\pi-\tilde{x}^{(i)}|>\epsilon)=0$, thus $\mathbb{P}_{\pi}(d(\pi)=\hat{d}(\pi))=1$.
        \item  It holds that $\mathbb{P}_{\pi_x}(d(\pi_x)=1\mid A)=1$ because the reasons rule applies, which implies that $\mathbb{P}_{\pi_x}(d(\pi_x)=1\cap B\cap A)=\mathbb{P}_{\pi_x}( B\cap A)$.
        \item As a probability, $\mathbb{P}_{\pi_x}(B)$ always lies between $0$ and $1$, and by our assumption is strictly greater than $0$. Dividing by such a term can only increase the value of the fraction.
    \end{enumerate}

\end{proof}

\retheorashomon*

\begin{proof}
First, we show that the bound holds. We know that $d(x^{(i)})\in\lbrace 0,1\rbrace$, thus, we prove this by distinguishing two cases. (I) $d(x^{(i)})=0$, then 
\begin{equation*}
    \mathbb{P}_{H_{D,\delta}}(d(x^{(i)})\neq\hat{d}(x^{(i)}))=
    \mathbb{P}_{H_{D,\delta}}(\hat{d}(x^{(i)})=1)\geq \text{min}\lbrace \mathbb{P}_{H_{D,\delta}}(\hat{d}(x^{(i)})=1), \mathbb{P}_{H_{D,\delta}}(\hat{d}(x^{(i)})=0)\rbrace.
\end{equation*}
(II) $d(x^{(i)})=1$, then 
\begin{equation*}
    \mathbb{P}_{H_{D,\delta}}(d(x^{(i)})\neq\hat{d}(x^{(i)}))=
    \mathbb{P}_{H_{D,\delta}}(\hat{d}(x^{(i)})=0)\geq \text{min}\lbrace \mathbb{P}_{H_{D,\delta}}(\hat{d}(x^{(i)})=1), \mathbb{P}_{H_{D,\delta}}(\hat{d}(x^{(i)})=0)\rbrace.
\end{equation*}
This proves the bound. 

Now we show asymptotic consistency, define
\begin{equation*}
    \hat{\pi}_1:=\frac{1}{|RS_{\hat{p}}|}\sum_{m\in RS_{\hat{p}}}\mathbf{1}_{\hat{d}_m(x^{(i)})=1},
    \qquad
    \hat{\pi}_0:=\frac{1}{|RS_{\hat{p}}|}\sum_{m\in RS_{\hat{p}}}\mathbf{1}_{\hat{d}_m(x^{(i)})=0}.
\end{equation*}
Then
\begin{equation*}
    \hat{\theta}_{RS_{\hat{p}}}=\min\{\hat{\pi}_1,\hat{\pi}_0\}.
\end{equation*}

By the Law of Large Numbers, $\hat{\pi}_1$ and $\hat{\pi}_0$ converge in probability as $|RS_{\hat{p}}|\to\infty$, i.e.
\begin{equation*}
    \hat{\pi}_1 \xrightarrow[]{p} \mathbb{P}_{H_{D,\delta}}(\hat{d}(x^{(i)})=1),
    \qquad
    \hat{\pi}_0 \xrightarrow[]{p} \mathbb{P}_{H_{D,\delta}}(\hat{d}(x^{(i)})=0).
\end{equation*}
Since the minimum operator is continuous, the Continuous Mapping Theorem \citep{van2000asymptotic} implies
\begin{equation*}
    \hat{\theta}_{RS_{\hat{p}}}
    =\min\{\hat{\pi}_1,\hat{\pi}_0\}
    \xrightarrow[]{p}
    \min\{\mathbb{P}_{H_{D,\delta}}(\hat{d}(x^{(i)})=1),
    \mathbb{P}_{H_{D,\delta}}(\hat{d}(x^{(i)})=0)\}
    =\theta.
\end{equation*}
Hence, $\hat{\theta}_{RS_{\hat{p}}}$ is asymptotically consistent.

Now we show that the proposed estimator is biased downwards, i.e. $\mathbb{E}_{{H_{D,\delta}}}[\hat{\theta}_{RS_{\hat{p}}}]\leq\theta$:
    \begin{align*}
        \mathbb{E}_{H_{D,\delta}}[\hat{\theta}_{RS_{\hat{p}}}]&=\mathbb{E}_{{H_{D,\delta}}}[\text{min}\lbrace\frac{1}{|RS_{\hat{p}}|}\sum\limits_{m\in RS_{\hat{p}}}\mathbf{1}_{\hat{d}_m(x^{(i)})=1},\frac{1}{|RS_{\hat{p}}|}\sum\limits_{m\in RS_{\hat{p}}}\mathbf{1}_{\hat{d}_m(x^{(i)})=0}\rbrace]\\
        &\overset{*}\leq\text{min}\lbrace\mathbb{E}_{{H_{D,\delta}}}[\frac{1}{|RS_{\hat{p}}|}\sum\limits_{m\in RS_{\hat{p}}}\mathbf{1}_{\hat{d}_m(x^{(i)})=1}],\mathbb{E}_{{H_{D,\delta}}}[\frac{1}{|RS_{\hat{p}}|}\sum\limits_{m\in RS_{\hat{p}}}\mathbf{1}_{\hat{d}_m(x^{(i)})=0}]\rbrace\\
        &=\text{min}\lbrace\frac{1}{|RS_{\hat{p}}|}\sum\limits_{m\in RS_{\hat{p}}}\mathbb{E}_{{H_{D,\delta}}}[\mathbf{1}_{\hat{d}_m(x^{(i)})=1}],\frac{1}{|RS_{\hat{p}}|}\sum\limits_{m\in RS_{\hat{p}}}\mathbb{E}_{{H_{D,\delta}}}[\mathbf{1}_{\hat{d}_m(x^{(i)})=0}]\rbrace\\
        &\overset{**}=\text{min}\lbrace\mathbb{E}_{{H_{D,\delta}}}[\mathbf{1}_{\hat{d}(x^{(i)})=1}],\mathbb{E}_{{H_{D,\delta}}}[\mathbf{1}_{\hat{d}(x^{(i)})=0}]\rbrace\\
        &=\text{min}\lbrace \mathbb{P}_{H_{D,\delta}}(\hat{d}(x^{(i)})=1), \mathbb{P}_{H_{D,\delta}}(\hat{d}(x^{(i)})=0)\rbrace\\
        &=\theta.
    \end{align*}
    
    In $*$, we applied Jensen's inequality, and in $**$, we used that all Rashomon models share the same expected value as $\hat{p}$ over $H_{D,\delta}$.
\end{proof}

\retheoepcontest*

\begin{proof}
    We know that $\hat{d}(\tilde{x}^{(i)})\neq\hat{d}(x^{(i)})$. Since the decision model $\hat{d}$ is epistemically reliable, we can infer that $d(x^{(i)})=\hat{d}(x^{(i)})$. Thus, we can infer that $\hat{d}(\tilde{x}^{(i)})\neq d(x^{(i)})$.
\end{proof}

\retheofriendly*

\begin{proof}
    Before proving the theorem, we define a function that will be useful in the proof:
    \begin{equation*}
        q': \mathcal{X}\times E\rightarrow [0,1]; (x,e)\mapsto \mathbb{P}(Y=1\mid X=x).
    \end{equation*}
    Now, we can prove the theorem:
    \begin{align*}
        \mathbb{E}_{Y,X,E}[d^*\neq d_{X,E}]&=\mathbb{E}_{Y,X,E}[d^*\neq \mathbf{1}_{p_{X,E}(X,E)>\tau}]\\
        &\overset{\text{(i)}}=\underset{q}{\text{min}}\;\mathbb{E}_{Y,X,E}[d^*\neq \mathbf{1}_{q(X,E)>\tau}]\\
        &\overset{\text{(ii)}}\leq\mathbb{E}_{Y,X,E}[d^*\neq \mathbf{1}_{q'(X,E)>\tau}]\\
        &\overset{\text{(iii)}}=\mathbb{E}_{Y,X}[\mathbb{E}_{E\mid Y,X}[d^*\neq \mathbf{1}_{q'(X,E)>\tau}\mid X,Y]]\\
        &\overset{\text{(iv)}}=\mathbb{E}_{Y,X}[\mathbb{E}_{E\mid Y,X}[d^*\neq \mathbf{1}_{p_X(X)>\tau}\mid X,Y]]\\
        &\overset{\text{(v)}}=\mathbb{E}_{Y,X}[d^*\neq \mathbf{1}_{p_X(X)>\tau}]\\
        &= \mathbb{E}_{Y,X}[d^*\neq d_{X}].
    \end{align*}

        In the derivation we used the following:
    \begin{enumerate}[(i)]
        \item $p_{X,E}$ is Bayes optimal and thus describes the function with minimal error in expectation over $Y,X,E$ mapping from $\mathcal{X}\times\mathcal{E}$ to $[0,1]$. 
        \item The minimum over all functions $q$ is smaller than any particular function $q'$.
        \item  The tower rule for expectations.
        \item By definition, $\mathbf{1}_{q(X,E)>\tau}=\mathbf{1}_{p_X(X)>\tau}$.
        \item $d^*\neq \mathbf{1}_{p_X(X)>\tau}$ is a random variable that is measurable with respect to the $\sigma-$algebra generated by $Y,X$. Therefore it can be drawn out of the conditional expectation value. 
    \end{enumerate}
\end{proof}

\section{Counterexamples}
\label{app:counterexamples}
For all counterexamples, we work with the same general setup. The decision threshold is given by $\tau=0.5$, the locality is defined as $\epsilon=0.1$, and the distance norm applied is the absolute distance $|\cdot|$. Moreover, for all counterexamples, we rely on the same data generating process and prediction model.

\textbf{Data generating process:} Let the data generating process be defined by $(X,Y)$, with $X$ a one-dimensional random variable uniformly distributed over the interval $[0,1]$ and $Y$ be defined as follows:
\begin{equation*}
    Y=\begin{cases}
        1 \qquad\text{ if } X+\mathcal{N}\geq 1\\
        0 \qquad\text{ else,}
    \end{cases}
\end{equation*}
where $\mathcal{N}$ is a noise term defined by a uniform distribution over the interval $[0,1]$. Then, $p(x)=\mathbb{P}(Y=1\mid X=x)=x$. Since $\tau=0.5$, the epistemically correct decision is given by:
\begin{equation*}
    d(x^{(i)})=\begin{cases}
        1 \qquad\text{ if } x^{(i)}\geq 0.5\\
        0 \qquad\text{ else.}
    \end{cases}
\end{equation*}

\textbf{Prediction model:} Let the prediction model be defined by:
\begin{equation*}
    \hat{p}(x):=\begin{cases}
        2x \qquad\text{ if }x<0.5\\
        2-2x \;\text{ else.}
    \end{cases}
\end{equation*}
Then, the actual decision is given by:
\begin{equation*}
    \hat{d}(x):=\begin{cases}
        1 \qquad\text{ if }0.25<x\leq 0.75\\
        0 \qquad\text{ else.}
    \end{cases}
\end{equation*}

Now, we have everything set up to construct three counterexamples. For each of the respective theorems, i.e.  \Cref{theo:Count}, \Cref{theo:LIME}, and \Cref{theo:anchor}, we construct cases where the conditions of the respective theorem are satisfied but the decision is not epistemically contestable.

\begin{counterexample}
    \label{ce:counterfactuals}
    Consider an individual with $x^{(i)}=0.2$ where the features are not subject to measurement error, i.e. $\tilde{x}^{(i)}=x^{(i)}$. Then, the individual receives a negative decision since $\hat{d}(0.2)=0$. A potential counterfactual within an $\epsilon=0.1$ environment is $x^{(i_c)}=0.25$ because $|0.2-0.25|<\epsilon$ and $\hat{d}(0.25)=1$. The continuity rule applies for the individual and the counterfactual since $d(0.2)=d(0.25)=0$.

    Consistent with \Cref{theo:Count}, the decision is somewhere contestable since the counterfactual lies in an $\epsilon$ environment around $x^{(i)}$ and $\hat{d}(0.25)=1\neq 0=d(0.25)$. 
    
    However, the decision itself is epistemically correct since $\hat{d}(0.2)=d(0.2)=0$.
\end{counterexample}

\begin{counterexample}
    \label{ce:LIME}
    Consider an individual with $x^{(i)}=0.6$ where the features are not subject to measurement error, i.e. $\tilde{x}^{(i)}=x^{(i)}$. Then the optimal LIME explanation for $\epsilon=0.1$ would be $g(x)=2-2x$. Note, that the LIME explanation is locally perfectly faithful, i.e. $\text{LF}(g,\hat{p})=0$. In our example, the monotonicity rule applies positively in $x$, that means for any two values with $x\geq x'$, holds $p(x)\geq p(x')$.

    As shown in \Cref{theo:LIME}, we do find a somewhere inaccuracy, for example, $x=0.55$ lies within an $\epsilon$ environment around $x^{(i)}$ and it holds that $\hat{p}(0.55)=0.9\neq 0.55=p(0.55)$. 
    
    However, even in the given local environment, we do not find an epistemically incorrect decision since $\forall x \in \mathcal{B}_{0.1}(0.6): \hat{d}(x)=d(x)=1$.
\end{counterexample}

\begin{counterexample}
    \label{ce:anchors}
    Consider an individual with $x^{(i)}=0.55$ where the features are not subject to measurement error, i.e. $\tilde{x}^{(i)}=x^{(i)}$. The individual receives a positive decision since $\hat{d}(0.55)=1$. Let $B=[0.25,0.75[$, this constitutes an anchor with perfect precision $\text{prec}(B)=1$ because in this interval, the model always decides positively. Let $A=[0,0.5[$ be the domain of the reason rule for a negative decision. The reason rule applies as for any $x\in A: d(x)=0$. Also, $\mathbb{P}_{\pi_{x^{(i)}}}(A\cap B)=1>0=\delta$.

    As shown in \Cref{theo:anchor}, we do find a decision in the neighborhood that is epistemically contestable, for example, $x=0.49$ lies within an $\epsilon$ environment around $x^{(i)}$ and it holds that $\hat{d}(0.49)=1\neq 0=d(0.49)$.

    However, for the individual in question, the decision is epistemically correct since $\hat{d}(0.55)=d(0.55)=1$.
\end{counterexample}

\section{Details Figure \ref{fig:violations}}
\label{app:details-examples}

All code needed to reproduce the results in Figure \ref{fig:violations} will be made publicly available upon publication; an anonymized version is made available to the reviewers (link in footnote).\footnote{\url{https://anonymous.4open.science/r/contestability-code-334B/README.md}}
The examples were coded in Python $3.11.7$ and run on an M3 Pro Chip. Running the code takes less than one hour.
Details on packages and versions are provided in the code repository.

\paragraph{Dataset}{
For the examples, we used the German Credit dataset (available at the UCI Machine Learning Repository \citep{south_german_credit_573}, obtained from OpenML \citep{openml}). It contains $1000$ samples with $20$ features (a mix of categorical and numerical variables such as requested credit amount, duration, employment status, etc.). The prediction target is binary credit risk classification: Good Credit ($1$) vs Bad Credit ($0$), with a class distribution of $700$ good and $300$ bad cases. All categorical features were converted to integers via LabelEncoder during preprocessing.
}

\paragraph{Models}{
The examples are based on two models: A deep decision tree for LIME and counterfactuals, and a shallow decision tree for anchors. 
Both the shallow and deep decision trees were fit using scikit-learn \citep{scikit-learn}.
The data was split into 80\% training and 20\% test sets using stratified sampling. Categorical features were encoded using label encoding prior to training. The \emph{shallow decision tree} was constrained to a maximum depth of $3$, while the \emph{deep decision tree} was grown without depth constraints (allowing it to fully expand). Both models used default values for \texttt{min\_samples\_split} (2) and \texttt{min\_samples\_leaf} (1). The shallow tree was designed to produce simple, human-interpretable decision rules, whereas the unconstrained deep tree serves as a baseline representing a more complex, overfitting model.
The shallow decision tree (max\_depth=3) achieves an accuracy of 73.0\% with a false negative rate of 7.1\%, while the deep decision tree (max\_depth=16) achieves an accuracy of 68.0\% with a false negative rate of 26.4\%.
}

We computed three explanation methods on these examples: counterfactuals, LIME, and Anchors. The results are visualized in the main paper in Figure \ref{fig:violations}.

\paragraph{Counterfactual}{
Counterfactual explanations were generated using a brute force sparse counterfactual search. 
That is, we exhaustively search over one-dimensional changes to the explained instance. The values are taken from the empirical distribution of each feature in the data. 
We return the counterfactual with the minimal normalized Euclidean distance, which is the Euclidean distance after normalizing the features.
This method finds the closest one-dimensional counterfactual.\\
We computed counterfactual explanations for all rejected applicants and report the one with the closest counterfactual. The concerned applicant is recommended to increase the requested amount by just $38$ from $1344$ to $1382$, which flips the model's decision. The normalized Euclidean distance is only $0.0137$, meaning that the counterfactual is only $0.0137$ standard deviations away from the factual.
We report all feature values in the code repository.\\
We note that the counterfactual does not only expose a continuity conflict, like LIME below it, also exposes a monotonicity conflict: Increasing the requested credit amount leads to a more favorable decision.
}

\paragraph{LIME}{
LIME explanations were generated using the lime package and the tabular explainer in classification mode, with Chebyshev distance and kernel width $0.5$.\\
We compute LIME for the same applicant as in the counterfactual example and find that LIME attributes a positive weight of $0.0614$ to the requested credit amount.
According to Theorem \ref{theo:LIME}, a negative relationship but positive LIME weight indicates a monotonicity conflict.
From the counterfactual explanation we already know that the model rejects the applicant but accepts them when the requested credit amount is increased by $38$, confirming our analysis.
}

\paragraph{Anchor}{
The model anchor was extracted directly from the decision path of the shallow decision tree. For a given instance, the anchor consists of the conjunction of all splitting rules encountered along the path from the root to the leaf node that determines the prediction. For the instance at test index 1 (a false negative where the model predicts ``bad credit'' but the true label is ``good credit''), the decision path yields the following rule: $\texttt{checking\_status} \leq 1.5$ and $\texttt{duration} > 31.5$. This anchor has a support of 83 training samples, with the model predicting the negative class for 100\% of these samples. As the reason rule, we consider: Applicants with $\texttt{duration} > 31$ and $\texttt{savings\_status} \geq 3$ should be accepted. The rule has a support of $29$ training samples. The overlap between the model anchor and the prior knowledge rule contains $10$ training samples. Thus, there is a clear conflict between the rule and the anchor.
}

\end{document}